\documentclass[12]{article}
\usepackage{amssymb, amsmath, amsthm, mathtools,bbold}
\usepackage[utf8]{inputenc}

    \usepackage{geometry}
    \usepackage{parskip}

  \usepackage{amsmath,amssymb,amscd}
\usepackage{amsfonts,graphicx,psfrag,esint}

\renewcommand{\Re}{\mathbb{R}}

\newcommand{\vph}{\vphantom{A^{A}_{A}}}
\newcommand{\sst}{\,\mid\,}

\newcommand{\dbydp}[2]{\frac{\partial #1}{\partial #2}}

\newcommand{\eqnref}[1]{(\ref{eqn#1})}

\newcommand{\uhat}{\hat{u}}

\newcommand{\xhat}{\hat{x}}

\newcommand{\Ahat}{\hat{A}}

\newcommand{\Rhat}{\hat{R}}

\newcommand{\rhohat}{\hat{\rho}}

\newcommand{\phihat}{\hat{\phi}}

\newcommand{\Phihat}{\hat{\Phi}}

\newcommand{\Omegahat}{\hat{\Omega}}

\newcommand{\bbR}{\mathbb{R}}

\newcommand{\ubar}{{\bar{u}}}

\newcommand{\Lone}{{L^1}}
\newcommand{\lone}[1]{\norm{#1}_\Lone}

\newcommand{\Linf}{{L^\infty}}
\newcommand{\linf}[1]{\norm{#1}_{\Linf}}

    \newtheorem{theorem}{Theorem}[section]
    \newtheorem{corollary}[theorem]{Corollary}
    \newtheorem{lemma}[theorem]{Lemma}

    \newtheorem{definition}[theorem]{Definition}
    
    \newtheorem{algorithm}[theorem]{Algorithm}

    \theoremstyle{spacedremark}
    \newtheorem{remark}[theorem]{Remark}

    \usepackage{hyperref,url}

    \newcommand{\set}[1]{\left\{#1\right\}}
    \newcommand{\smid}{\,\middle|\,}

    \newcommand{\RR}{\mathbb{R}}

    \newcommand{\abs}[1]{\ensuremath{\left\lvert #1 \right\rvert}}

    \DeclareMathOperator*{\argmin}{arg\,min}

    \newcommand{\norm}[1]{\ensuremath{\left\lVert #1 \right\rVert}}
    
    \newcommand{\grad}{\nabla}

    \DeclareMathOperator{\supp}{supp}
    
    \def\XXint#1#2#3{{\setbox0=\hbaox{$#1{#2#3}{\int}$ }
    \vcenter{\hbox{$#2#3$ }}\kern-.6\wd0}}

    \newcommand{\boundary}{\partial}

\usepackage{xcolor}
\usepackage[title]{appendix}
\usepackage{subfig}

\usepackage{algorithm}

\usepackage{fullpage}
\usepackage{mathtools}

\usepackage{amsmath,amssymb,amsthm}
\usepackage{graphicx}
\usepackage{hyperref}

\newcommand{\tim}[1]{}

\newcommand{\Lonea}{{L^1_\alpha}}
\newcommand{\lonea}[1]{\norm{#1}_\Lonea}
\newcommand{\Ltwoa}{{L^2_\alpha}}
\newcommand{\ltwoa}[1]{\norm{#1}_\Ltwoa}

\usepackage{eucal}
\newcommand{\R}{\mathcal{R}}
\newcommand{\M}{\mathcal{M}}
\renewcommand{\H}{\mathcal{H}}

\title{Weighted Cheeger-Buser Inequalities, with Applications to
Cutting Probability Densities - as Easy as 1,2,3}
\author{
Timothy Chu \\
  CMU\\
  \texttt{timothyzchu@gmail.com}\\
\and
Gary L.~Miller\thanks{Supported in part
        by National Science Foundation Grant CCF--1637523.}\\
  CMU\\
  \texttt{glmiller@cs.cmu.edu}\\
\and
Noel J.~Walkington\thanks{Supported in part
        by National Science Foundation Grant DMS--1729478.}\\
  CMU\\
  \texttt{noelw@cmu.edu}\\
\and
Alex L.~Wang\\
  CMU\\
  \texttt{alw1@cs.cmu.edu}
}

\begin{document}

\maketitle

\setcounter{page}{0}
\begin{abstract}

  In this paper, we show how sparse or isoperimetric cuts of a probability density function relate to Cheeger cuts of its principal eigenfunction, for appropriate definitions of `sparse cut' and `principal eigenfunction'.

  We construct these appropriate definitions of sparse cut and
  principal eigenfunction in the probability density setting.
  Then, we prove Cheeger and Buser type inequalities similar to
  those for the normalized graph Laplacian of Alon-Milman.  We
  demonstrate that no such inequalities hold for most prior definitions of sparse cut and principal eigenfunction.  We apply this result to generate novel algorithms for cutting probability densities and clustering data, including a principled variant of spectral clustering.

\end{abstract}

\pagebreak
\section{Introduction}

Clustering, the task of
partitioning data into groups, is one of the fundamental tasks
in machine learning~\cite{bishopBook, NgSpectral01, e96, kk10clustering, sw18}.
The data that practitioners seek to cluster is commonly modeled as i.i.d. samples from a probability density function,
an assumption foundational in theory, statistics, and AI~\cite{bishopBook,
TrillosContin16, bc17focs, mv10, Ge2018}. 
A clustering algorithm on data drawn from a probability density function
should ideally converge to a partitioning of the
probability density function, as the number of i.i.d. samples grows
large~\cite{von2008consistency, TrillosVariational15,
TrillosContin16}.
 

In this paper, we outline a new strategy for
clustering, and make progress on implementing it. We
propose the following two questions to help generate new clustering
algorithms:

\begin{enumerate}
  \item How can we partition probability density functions? How can we
    do this so
    that two data points drawn from the same part of the partition are
    likely to be similar, and two data points drawn from different parts of the partition
    are likely to be dissimilar?

  \item What clustering algorithms converge to such a partition, 
    as the number of samples from the density function grows large?
\end{enumerate}

In this paper, we address the first point, and make partial progress on
the second. We
focus on the special case of $2$-way partitioning, which can be seen as
finding a good cut on the probability density function. First, we
propose a new notion of sparse (or isoperimetric)
cuts on density functions. We call this an $(\alpha, \beta)$-sparse cut,
for real parameters $\alpha$ and $\beta$. Next, we propose a new notion of spectral
sweep cuts on probability densities, called a $(\alpha,
\gamma)$-spectral sweep cut, for real parameters $\alpha$ and $\gamma$.
We show that a $(\alpha, \gamma)$-spectral sweep cut provably
approximates an $(\alpha, \beta)$-sparse cut when
$\beta = \alpha+1$ and $\gamma=\alpha+2$. In particular, $\alpha = 1, \beta = 2,
\gamma=3$ is such a setting.
Our result holds for any $L$-Lipschitz
probability density function on $\mathbb{R}^d$, for any $d$. 
Based on past success applying sparse graph cuts to dividing
data into two pieces in the machine learning setting~\cite{grady2006isoperimetric, belkin2004semisup}, we believe our
similarly defined $(\alpha=1,\beta=2)$-sparse cuts
may have similar ideal behavior when it comes to partitioning our density
function.

To our knowledge, this is the first spectral method of
cutting probability density functions that has any theoretical guarantee
on the cut quality.  
The key mathematical contribution of this paper is a new Cheeger and Buser
inequality for probability density functions, which we use to prove
that $(\alpha,\gamma)$-spectral sweep cuts approximate $(\alpha,
\beta)$-sparse cuts on probability
density functions for the aforementioned settings of $\alpha, \beta, \gamma$.
These inequalities
are inspired by the Cheeger and Buser inequalities on graphs and
manifolds~\cite{AlonM84, Cheeger70, Buser82}, which have received
considerable attention in graph algorithms and machine
learning~\cite{ChungBook97, SpielmanTeng2004, Orecchia08, Orecchia2011,
kw16, belkin2004semisup}. These new inequalities do not directly follow
from either the graph or manifold Cheeger inequalities,
something we detail in Section~\ref{sec:cheeger-buser-difficulties}.
We note that our Cheeger and Buser inequalities for probability density
functions require a careful definition of eigenvalue and
sparse/isoperimetric cut: existing definitions
lead to false inequalities. 

Finally, our paper will present a discrete $2$-way clustering algorithm that we
suspect converges to the $(\alpha=1, \gamma=3)$-spectral sweep cut as the
number of data points grows large.  Our
algorithm bears similarity to classic
spectral clustering methods, although we note that classical spectral
clustering does not have any theoretical guarantees on the cluster quality. We note
that we do not prove convergence of our discrete clustering method to
the $(\alpha=1, \gamma=3)$-spectral sweep cut, and leave this for future work.

\subsection{Definitions}\label{sec:definitions}
In this subsection, we define $(\alpha, \beta)$ sparsity, $(\alpha,
\gamma)$ eigenvalues/Rayleigh quotients, and $(\alpha, \gamma)$ sweep
cuts. 

\vspace{2 mm}

\begin{definition} Let $\rho$ be a probability density function with domain
  $\mathbb{R}^d$, and let $A$ be a
  subset of $\mathbb{R}^d$.

  The \textbf{$(\alpha, \beta)$-sparsity} of the cut defined by $A$ is the
  $(d-1)$ dimensional integral of
  $\rho^\beta$ on the cut, divided by the $d$ dimensional integral of
  $\rho^\alpha$ on the side of the cut where this integral is
  smaller.

\end{definition}
\vspace{2 mm}

\begin{definition}
The \textbf{$(\alpha, \gamma)$-Rayleigh quotient} of $u$ with
respect to $\rho:\Re^d \to \Re_{\geq 0}$ is:

\[
  R_{\alpha, \gamma}(u) \coloneqq \frac{\int_{\Re^d} \rho^\gamma
  |\nabla u|^2}{\int_{\Re^d} \rho^{\alpha}|u|^2} 
\]

A \textbf{$(\alpha, \gamma)$-principal eigenvalue} of $\rho$ is
$\lambda_2$, where:

\[ \lambda_2 := \inf_{\int \rho^\alpha u = 0} R_{\alpha, \gamma}(u). \]

Define a \textbf{$(\alpha, \gamma)$-principal eigenfunction} of
$\rho$ to be a function $u$ such that $R_{\alpha, \gamma}(u) =
\lambda_2$.
\end{definition}
\vspace{2 mm}

Now we define a sweep cut for a given function with respect to a
a positive valued function supported on $\mathbb{R}^d$:

\vspace{2 mm}
\begin{definition} Let $\alpha, \beta$ be two real numbers, and $\rho$ be
  any function from $\Re^d$ to $\Re_{\geq 0}$.
  Let $u$ be any function from $\Re^d \to \Re$, and let $C_t$ be the cut
  defined by the set $\{s \in \Re^d | u(s) > t\}$. 
  
  The \textbf{sweep-cut} algorithm
  for $u$ with respect to $\rho$ returns the cut $C_t$ of minimum $(\alpha,
  \beta)$ sparsity, where this sparsity is measured with respect to $\rho$.

When $u$ is a $(\alpha, \gamma)$-principal eigenfunction, the sweep cut is called
a \textbf{$(\alpha, \gamma)$-spectral sweep cut} of $\rho$.  
\end{definition}

\textbf{Additional Definitions:}

A function $\rho: \Re^d \to \Re_{\geq 0}$ is
\textbf{$L$-Lipschitz} if $|\rho(x)-\rho(y)|_2 \leq L|x-y|_2$ for
all $x, y \in \Re^d.$

A function is $\rho:\Re^d \to \Re_{\geq 0}$ is
\textbf{$\alpha$-integrable} if $\int_{\Re^d} \rho^\alpha$ is
well defined and finite. Throughout this paper, we assume $\rho$
is always $\alpha$-integrable.

\subsection{Past Work}\label{sec:past-work}

\textbf{Spectral Clustering and Sweep Cut Algorithms on Data}

The spectral clustering algorithms of Shi and Malik~\cite{ShiMalik97}
and those of Ng, Jordan, and Weiss~\cite{NgSpectral01} are some of the
most popular clustering algorithms on data (over 10,000 citations).
If we want to split data points into two clusters, their algorithm works
as follows: for $n$ data points, compute an $n \times n$ matrix $M$ on
the data, and compute the principal eigenvector $e$ of the matrix. Then, find a
threshold value $t$ such that all points $p$ where $e(p) \leq t$
are considered to be on one side of the cut, and all other points
where $e(p)  > t$ are on the other. Often, the matrix is a
Laplacian matrix of some graph built from the
data~\cite{von2007tutorial}.  

Von Luxburg, Belkin, and Bosquet~\cite{von2008consistency} proved that
if the data is modeled as $n$ i.i.d samples from a probability density
$\rho$, the matrix $M$ is a Laplacian matrix with certain structural
assumptions, and certain regularity
assumptions on $\rho$ hold, then classical spectral clustering
algorithms converge to a $(\alpha = 1, \gamma = 2)$-spectral sweep cut
on $\rho$~\footnote{We note that these authors used different terminology to describe this
result, as their papers did
not define $(\alpha, \gamma)$-spectral sweep cuts.}. These results were refined in~\cite{rosasco2010learning,
TrillosRate15, TrillosVariational15}. We note that there are no sparsity
guarantees known for a $(\alpha = 1, \gamma=2)$-spectral sweep cut, and
we show $1$-Lipschitz examples of $\rho$ where this spectral sweep cut leads to undesirable
behavior.

%
%
%

\textbf{Cheeger Inequality and Sparse Cuts in Graphs}

In 1984, Alon and Milman discovered a graph Cheeger
inequality~\cite{AlonM84}, which showed that a graph spectral sweep cut
is approximately sparse. For a formal definition of sparse (or
isoperimetric) cuts in a
graph theoretic setting, see~\cite{AlonM84}. The graph Cheeger inequality has guided decades of
algorithmic and theoretical research on graph partitioning,
random walks, and spectral graph theory in general~\cite{ChungBook97, kw16, Orecchia08, Orecchia2011,
  SpielmanTeng2004}.

The graph Cheeger inequality implies that partitioning a graph
based on the principal graph eigenvector (via spectral-sweep
    methods) will find a provably sparse
cut~\cite{ChungBook97, Orecchia2011, Orecchia08, Louis12, Lee2014}.  It also implies that a
slowly mixing random walk is likely to yield good information
about a sparse graph cut; this intuition was leveraged in Spielman and
Teng's seminal nearly-linear time algorithm on graph partitioning
in~\cite{SpielmanTeng2004}. Cheeger's inequality for graphs has
been an inspiration for decades of spectral graph theory
research (for more information, see~\cite{ChungBook97}).

Sparse cuts have been researched extensively in graph theory
~\cite{leightonrao88, arv04, chawla05sparse, Andersen09,
madry10fast, Louis12, kw16}. Sparse cuts, and the related notion of
balanced cuts, have
been used to generate fast multicommodity flow
algorithms~\cite{leightonrao88, arv04}.  
These cuts are deeply related to expander
decomposition~\cite{SpielmanTeng2004, wulff17expander, sw19expander},
which in turn has proven immensely useful in algorithm design. For a
brief overview of the application of sparse and balanced cuts in
computer science theory, see the
introduction of~\cite{madry10fast}. For a survey of
uses of expander decomposition, see the introduction of
~\cite{sw19expander}. 

We note that expander decompositions are based on the idea that sparse
cuts form natural partitions of the graph into clustered components.
This intuition is leveraged in machine
learning~\cite{belkin2004semisup}.
\textbf{Cheeger's Inequality and Sparse Cuts for Manifolds}

The Cheeger inequality on a manifold states that the fundamental
eigenvalue $\lambda_2$ of the Beltrami-Laplace operator is bounded below by
the square of the sparse or isoperimetric cut, divided by $4$
(See~\cite{Cheeger70} for details).  In 1982, Buser proved an
upper bound for $\lambda_2$ provided that the manifold has
lower bounded Ricci curvature~\cite{Buser82}. In contrast to the graph case, this inequality
is false without the curvature assumption. This inequality has
been widely used in the theory of differential
geometry~\cite{vsc08, hk00}. Intuition based on manifold Cheeger and
manifold Buser
has been used to great effect in
semi-supervised learning and image processing~\cite{belkin2004semisup,
grady2006isoperimetric}.

For a formal definition of sparse or isoperimetric cut on manifolds,
see~\cite{Cheeger70} or~\cite{Buser82}.

\textbf{The Cheeger-Buser Inequality and Sparse Cuts for Convex Bodies and
Density Functions}

The Cheeger-Buser inequality, and the related notion of sparse cuts,
have seen renewed interest in the computer science
literature ~\cite{Lovasz90, Lee18, Milman09}. 
Past
researchers have used ideas inspired by sparse cuts and the
Cheeger-Buser
inequality to recover fast mixing time
lemmas for random walks for log-concave density functions supported on
convex polytopes~\cite{Lovasz90, Dyer91, KLS95, Lee18, Gromov83}. These
lemmas have in turn been used to find fast sampling algorithms from
such density functions~\cite{Lee18}, and more.  

One prominent use of sparse cuts in convex geometry is the celebrated
Kannan-Lovasz-Simonovits Conjecture (KLS
Conjecture). This conjecture asserts that, for a convex body, the sparsity of the
sparsest hyperplane cut is a dimension-less constant approximation of
the sparsity of the optimal sparse cut~\cite{KLS95, Lee18survey}.
Generally, these results all
implicitly use the settings $(\alpha = 1, \beta = 1, \gamma =
1)$ for their definition of sparse cuts, eigenvalues, and eigenvectors. They operate in the setting where there is significant structure on the probability, such as being a log-concave
distribution~\cite{Lee18}. We note that the Cheeger-Buser inequality
fails for this setting of $\alpha, \beta, \gamma$ when the probability
density is Lipschitz but not log-concave. For a survey on uses of Cheeger's inequality in convex polytope theory
and log-concave
density function sampling, see~\cite{Lee18survey}.

\textbf{Discrete Machine Learning from Continuous Methods}

Our overall approach to clustering follows the line of work generating
discrete machine learning methods by analyzing its behavior in the
limit. This approach was used fruitfuilly by Su, Boyd, and
Candes~\cite{SuNesterov14} and Wibisono, Wilson, and
Jordan~\cite{WibisonoGradient16} to generate faster gradient descent
variants based on continuous time differential equations. For more information,
refer to their respective papers.

\subsection{Contributions}

Our paper has three core contributions: 
\begin{enumerate}
  \item A natural method for cutting probability density
    functions, based on a new notion of sparse cuts on density functions.
  \item A Cheeger and Buser inequality for Lipschitz probability density
    functions, and
  \item A clustering algorithm operating on samples, which heuristically
    approximates a spectral sweep cut on the density function when the
    number of samples grows large.
\end{enumerate}
We emphasize that our primary contributions are points 1 and 2, which
are formally stated in
Theorems~\ref{thm:sweep-cut} and~\ref{thm:Cheeger-Buser} respectively.  
Our clustering algorithm on samples, which is designed to approximate
the $(\alpha, \gamma)$-spectral sweep cut on the density function as the number of samples
grows large, is of secondary importance. 

We now state our two main theorems. 

\begin{theorem} \label{thm:sweep-cut} 
  \textbf{Spectral Sweep Cuts give Sparse Cuts:}

  Let $\rho:\Re^d \to \Re_{\geq 0}$ be an $L$-Lipschitz probability
  density function, and let $\alpha = \beta - 1 = \gamma - 2$.
  
  The $(\alpha,\gamma)$-spectral sweep cut of $\rho$ has 
  $(\alpha, \beta)$ sparsity $\Phi$ satisfying:
  \[ 
  \Phi_{OPT} \leq \Phi \leq O(\sqrt{dL\Phi_{OPT}} ).
  \]
  Here, $\Phi_{OPT}$ refers to the optimal $(\alpha,\beta)$ sparsity of
  a cut on $\rho$. 
\end{theorem}

In words, the spectral sweep cut of the
  $(\alpha, \gamma)$  eigenvector gives a provably good approximation to
  the sparsest $(\alpha, \beta)$ cut, as long as $\beta = \alpha+1$ and
  $\gamma = \alpha+2$.  
  Proving this result is a straightforward application of two new
  inequalities we present, which we will refer to as as the Cheeger and Buser inequalities
  for probability density functions.

\begin{theorem}\label{thm:Cheeger-Buser}
  \textbf{Probability Density Cheeger and Buser:}

  Let $\rho:\mathbb{R}^d \rightarrow \mathbb{R \geq 0}$ be an $L$-Lipschitz
  density function. Let $\alpha = \beta - 1 = \gamma - 2$.

  Let $\Phi$ be the infimum $(\alpha,\beta)$-sparsity of a cut through
  $\rho$, and let $\lambda_2$ be the $(\alpha,\gamma)$-principal eigenvalue of
  $\rho$. Then:
  \[ \Phi^2/4 \leq \lambda_2 \]
  and 
  \[\lambda _2 \leq O_{\alpha, \beta}(d \max(L \Phi, \Phi^2)).\]
  The first inequality is \textbf{Probability Density Cheeger}, and
  the second inequality is \textbf{Probability Density Buser}.
\end{theorem}
Note that we don't need $\rho$ to have a total mass of $1$ for any
of our proofs. The overall probability mass of $\rho$ can be arbitrary.

We note that Theorem~\ref{thm:Cheeger-Buser} has a partial converse:
\begin{lemma} \label{lem:converse} If $\alpha + \gamma > 2\beta$ or $\gamma -1 < \beta$, then the
  Cheeger-Buser inequality in Theorem~\ref{thm:Cheeger-Buser} does not
  hold.
\end{lemma}
In particular, if $\alpha = \gamma =1$ or $\alpha = 1, \gamma =2$, no
Cheeger-Buser inequality can hold for any $\beta$. These settings of
$\alpha$ and $\gamma$ encompass most past work on spectral cuts and
Cheeger-Buser inequalities on probability denstiy functions.

Finally, we give a discrete algorithm~\textsc{1,3-SpectralClustering} for clustering data points into
two-clusters.
We conjecture, but do not prove, that
\textsc{1,3-SpectralClustering} converges to the the $(\alpha=1,
\gamma=3)$-spectral sweep cut of the probability density function $\rho$
as the number of samples grows large.

\begin{algorithm}
  \textsc{1,3-SpectralClustering}

  \textbf{Input:} Point $s_1, \ldots s_n \in \mathbb{R}^d$, and similarity measure
  $K:\mathbb{R}^d, \mathbb{R}^d\rightarrow \mathbb{R}$.
  \begin{enumerate}
    \item Form the affinity matrix $A \in \mathbb{R}^{n \times n}$,
      where $A_{ij} = e^{-n^{2/d}\|s_i - s_j\|^2}$ for $i \not= j$ and $A_{ii} = 0$ for
      all $i$.
    \item Define $D$ to be the diagonal matrix whose $(i, i)$ element is
      the sum of $A$'s $i^{th}$ row. Let $L$ be the Laplacian formed
      from the adjacency matrix $D^{1/2}AD^{1/2}$.
    \item Let $u$ be the principal eigenvector of $L$. Find the value
      $t$ where $t := \argmin_s \Phi_{\{u(v) > t\}}$, where $\Phi_S$ is
      the graph conductance of the cut defined by set $S$.
  \end{enumerate}
  \textbf{Output:} Clusters $G_1 = \{v : u(v) > t\}, G_2 = \{v : u(v)
  \leq t\}$.
\end{algorithm}
We note that this resembles the unnormalized spectral
clustering based on the work of Shi and Malik~\cite{ShiMalik97} and Ng,
Weiss, and Jordan~\cite{NgSpectral01, TrillosVariational15}. The major difference is that we build our
Laplacian from the matrix $D^{1/2}AD^{1/2}$. Past work on spectral
clustering builds the Laplacian from the matrices $A$ or
$D^{-1/2}AD^{-1/2}$~\cite{von2007tutorial, TrillosVariational15}.

\subsection{Differences between our work and past work}\label{sec:cheeger-buser-difficulties}
%

  Our work differs from past work in the following key ways:

\begin{enumerate}
\item Our work differs from past practical work on spectral sweep cuts
cuts~\cite{TrillosRate15, TrillosVariational15, ShiMalik97,
  NgSpectral01}, as those methods perform what we call a $(\alpha = 1,
    \gamma = 2)$-sweep cut. These sweep cuts have no theoretical
    guarantees, much less a guarantee on their $(\alpha, \beta)$
    sparsity. Lemma~\ref{lem:converse} shows that no Cheeger and Buser inequality can
    simultaneously hold for any setting of $\beta$ when $\alpha = 1,
    \beta = 2$. 
    
    We will further show that using a $(\alpha=1,
    \gamma=2)$-sweep cut can lead to undesirable cuts of $1$-Lipschitz
    probability densities, with poor sparsity guarantees.

\item We note that probability density Cheeger-Buser is not easily implied by
  graph or manifold Cheeger-Buser. For a lengthier discussion on this,
    see Appendix~\ref{app:notgraph}.
\item We
do not require any assumptions on our probability density
except that it is Lipschitz. Past work on Cheeger-Buser inequality
for densities focused on log-concave distributions, or
mixtures thereof~\cite{Lee18survey, Ge2018}.  
\item For our work, the probability
density $\rho$ is not required to be bounded away from $0$.
This is a sharp departure from many existing results: past results on
partitioning probability densities required a positive lower bound on
$\rho$~\cite{von2008consistency, TrillosRate15}.
The strongest results in fields like linear elliptic partial
differential equations depend on $\rho$ being bounded away from
$0$ ~\cite{w17}.
\end{enumerate}
Our work is the first spectral sweep cut algorithm 
that guarantees a sparse cut on Lipschitz densities $\rho$, without requiring strong
parametric assumptions on $\rho$.

\subsubsection{Technical Contribution}
The key technical contribution of our proof is proving Buser's
inequality on Lipschitz probability densities via
mollification~\cite{s38, f44} with disks of varying radius. 
This paper is the first time mollification with disks of varying radius
have been used. We emphasize that the most difficult part of our paper
is proving the Buser inequality.

Mollification has a long history in mathematics dating back to Sergei
Sobolev's celebrated proof of the Sobolev embedding
theorem~\cite{s38}. It is one of the key tools in numerical
analysis, partial differential equations, fluid mechanics, and functional
analysis~\cite{f44, lw01, ss09, m03}, and analogs of mollification have been
used in computational complexity settings~\cite{dkn09}.  Informally
speaking, mollification is used to
create a series of smooth functions approximating a non-smooth
function,  by  convolving the original function with a
smooth function supported on a disk.  
Notably, an approach using convolution is used by Buser in~\cite{Buser82} to
prove the original Buser's inequality, albeit with an intricate
pre-processing step on any given cut.

To prove Buser's inequality on Lipschitz probability density functions $\rho$, we
will show that given a cut $C$ with low $(\alpha, \beta)$-sparsity, we can find a function $u$ with low $(\alpha, \gamma)$-Rayleigh
quotient. We build $u$ by starting with the indicator function $I_C$ for
cut $C$ (which is $1$ on one side of the cut and $0$ on the other).
Next, we mollify this function with disks of varying radius. In
particular, for each point $r$ in the domain of $\rho$, we
spread out the point mass $I_C(r)$ over a disk of radius
proportional to $\rho(r)L$, where L is the Lipschitz constant of $\rho$.
The resulting function $u$ obtained by `spreading out' $I_C$ will have
low $(\alpha, \gamma)$-Rayleigh quotient.

For all past uses of mollification, the disks on which the smooth
convolving function is supported (we call this the mollification
    disk) have the same radius throughout
the manifold. The use of a uniform radius disk is critical for most uses and proofs in mollification.  
Our contribution is to allow the disks to vary in
radius across our density.  This variation in radius allow us to
deal with functions that approach $0$ (and explains the importance
    of the density being Lipschitz). No mollification disks
centered anywhere in our probability density will intersect
the $0$-set of the density. This overcomes significant
hurdles in many results for functional analysis and PDEs, as
many past significant results related to partial
differential equations rely on having a positive lower
bound on the density~\cite{w17, TrillosRate15}.

Proving our Buser inequality using
mollification by disks of various radius requires a fairly
delicate proof with many pages of calculus. Our key technical
lemma is a bound on how the $l_1$ norm of a mollified function
when the mollification disks have various radius, which can be
found in Section~\ref{sec:key_lemma}.

\section{Paper Organization}

We prove the Buser inequality in Section~\ref{sec:higher_dim},
   via a rather extensive series of calculus computations.
Our proof relies on a key technical lemma, which is presented in
Section~\ref{sec:key_lemma}.  This is by far the most difficult part of
our proof. 

We prove the Cheeger inequality in Section~\ref{sec:cheeger}.
The proof in this section implies that the $(\alpha, \alpha+1)$
sparsity of the $(\alpha, \alpha+2)$ spectral sweep cut of a
probability density function $\rho$ is provably close to the
$(\alpha, \alpha+2)$ principal eigenvalue of $\rho$.  We note
that this inequality does not depend on the Lipschitz nature of
the probability density function.

In Section~\ref{sec:sweep_cut}, we prove
Theorem~\ref{thm:sweep-cut}, which shows that a $(\alpha,
    \alpha+2)$ spectral sweep cut has $(\alpha, \alpha+1)$
sparsity which provably approximates the optimal $(\alpha,
    \alpha+1)$ sparsity.

In Section~\ref{sec:examples}, we go over example $1$-D
distributions that show that either Cheeger or Buser inequality
must fail for past definitions of sparsity and eigenfunctions.
We will prove Lemma~\ref{lem:converse} in this section. 

In Section~\ref{sec:counterexample}, we show an example Lipschitz
probability density where the $(\alpha = 1, \gamma=2)$ spectral sweep
cut has bad $(1,\beta)$ sparsity for any $\beta < 10$, and will
lead to an undesirable cut (from a clustering point of view) on this density function.
This is important since the spectral clustering
algorithm of Ng et al~\cite{NgSpectral01} is known to converge
to a $(\alpha=1, \gamma=2)$ spectral sweep cut on the underlying
probability density function, as the number of
samples grows large~\cite{TrillosVariational15}. 

Finally, we state conclusions and open problems in Section~\ref{sec:conclusion}.

In the appendix, we note that the Cheeger and Buser inequalities for probability
densities are not easily implied by graph or manifold
Cheeger-Buser. We also provide a simplified version of Cheeger's and
Buser's inequality for probability densities, in the $1$-dimensional case. This
may make easier reading for those unfamiliar with technical
multivariable mollification.


\section{Buser Inequality for Probability Density Functions}
\label{sec:higher_dim}

The key idea to proving Buser's inequality is as follows: given 
$\rho:\mathbb{R}^d \rightarrow \mathbb{R}_{\geq 0}$, 
    and a cut $u$ where $u \subset \mathbb{R}^d$, we will build a
    function $u_{\theta}$ whose $(\alpha, \gamma)$-Rayleigh
    quotient is close to the $(\alpha, \beta)$ sparsity of $u$.

    Roughly speaking, $u_\theta$ is built by convolving $u$ at
    point $x$ with
    a unit-weight disk with radius proportional to $\rho(x)$.
    Thus, we are convolving $u$ with a disk, where the radius of
    the disk varies between points $x \in \mathbb{R}^d$, and the
    radius is directly proportional to $\rho$.

\subsection{Weighted Buser-type Inequality}

We now prove our weighted Buser-type inequality, from
Theorem~\ref{thm:Cheeger-Buser}. We state our result in
terms of general $(\alpha, \beta, \gamma)$.

\begin{theorem}
  \label{thm:buser_n}
  Let $\rho: \Re^d \to \Re_{\geq 0}$ be an $L$-Lipschitz function,
  $\lambda_2$ be a $(\alpha, \gamma)$-principal eigenvalue, and
  $\Phi$ the $(\alpha, \beta)$ isoperimetric cut.

  Then:
  \begin{align*}
    \lambda_2 \leq 3 \cdot 2^{\beta + 1} d \linf{\rho^{\gamma-\beta-1}}
    \max\left(L \Phi(A), 2^{\beta+1}\linf{\rho^{\alpha+1-\beta}} \Phi(A)^2 \vph\right)
\end{align*}
\end{theorem}

We note that when setting $(\alpha, \beta, \gamma) = (1,2,3)$,
the above expression simplifies into:

\begin{align*}
  \lambda_2 \leq 24 d 
  \max\left(L \Phi(A), 8 \Phi(A)^2 \vph\right).
\end{align*}



%

\subsection{Proof Strategy: Mollification by Disks of Radius
  Proportional to \texorpdfstring{$\rho$}{rho}}
To prove Theorem \ref{thm:buser_n} we construct an approximation $u_\theta$ of
$u$ for which the numerator and denominator of the Raleigh quotient, $R(u_\theta)$,
approximate respectively the numerator and denominator of this expression.
Specifically, $u_\theta$ will constructed as a mollification of $u$,
Recall the following two equivalent definitions of a mollification. They are equivalent 
by the change of variables $z = x-\theta \rho(x) y$.

\begin{equation} \label{eqn:utheta}
u_\theta(x) 
\coloneqq \int_{B(0,1)} \!\!\! u(x-\theta \rho(x) y) \phi(y) \, dy
= \int \! u(z) \phi_{\theta \rho(x)}(x-z) \, dz,
\quad \text{ where } \quad
\phi_{\eta}(z) 
= \frac{1}{\eta^d} \phi\left(\frac{z}{\eta}\right),
\end{equation}

with $\theta > 0$ a parameter to be chosen and $\phi:\Re^d
\rightarrow [0,\infty)$ a smooth radially symmetric function supported in the unit
open ball $B(0,1)=\{x \in \Re^d \;| \; |x| < 1\}$ with unit mass $\int_{\Re^d} \phi =
1$. When $\rho$ is constant it follows from the Tonelli theorem that
$\lone{u_\theta} = \lone{u}$; when $\rho$ is not constant the
following lemma shows that the latter still bounds the former.

\subsection{Key Technical Lemma: Bounding $L_1$ norm of a
  function with the $L_1$ norm of its
    mollification}\label{sec:key_lemma}
The following is our primary technical lemma, which roughly bounds the
$L_1$ norm of a mollified function $f$ by the $L_1$
norm of the original $f$. Here, the mollification radius is
determined by a function $\delta(x)$.
\begin{lemma} \label{lem:lonetheta} 

  Let
  $\delta:\RR^d\to\RR$ be Lipschitz continuous with Lipschitz constant
  $|\grad\delta(x)| \leq c < 1$ for almost every $x \in \bbR^d$. Let
  $\phi:\RR^d\to\RR_{\geq 0}$ be smooth, $\int_{\RR^d}\phi = 1$, and
  $\supp(\phi)\subseteq B(0,1)$.  Then
  \[
  \frac{1}{1+c}\norm{f}_{L^1}
  \leq \int_{\RR^d} \int_{B(0,1)} 
  \abs{f(x-\delta(x) y)}\phi(y)\,dy\,dx\leq
  \frac{1}{1-c}\norm{f}_{L^1},
  \qquad
  f \in L^1(\RR^d).
  \]
\end{lemma}

\begin{proof} (of Lemma~\ref{lem:lonetheta})
An application of Tonelli's theorem shows
\begin{align}
\int_{\RR^d} \int_{B(0,1)} \abs{f(x-\delta(x) y)}\phi(y)\,dy\,dx
=  \int_{B(0,1)}\phi(y)\int_{\RR^d} \abs{f(x-\delta(x)y)}\,dx\,dy.
\end{align}
Fix $y\in B(0,1)$ and consider the change of variables $z = x-\delta(x)y$. The
Jacobian of this mapping is $I - y \otimes \nabla \delta(x)$ which by Sylvester's determinant theorem has
determinant $1-y.\nabla \delta(x) > 0$. It follows that
\begin{align}
\int_{\RR^d} \int_{B(0,1)} 
\abs{f(x-\delta(x) y)}\phi(y)\,dy\,dx
=  \int_{B(0,1)}\phi(y)\int_{\RR^d} 
\frac{|f(z)|}{1-y.\nabla \delta(x)} \, dx \,dy,
\end{align}
and the lemma follows since $1-c \leq 1-y.\nabla \delta(x) \leq 1+c$.
\end{proof}
(Here, $a.b$ denotes the dot product between $a$ and $b$.)

We present the following simple corollaries, which is the primary
way our proof makes use of Lemma~\ref{lem:lonetheta}

\begin{corollary}\label{cor:lonetheta}
For any Lipschitz continuous function $\rho: \bbR^d \rightarrow \bbR^{\geq 0}$
with Lipschitz constant $L$ and any $\theta$ with $0 < \theta L < 1$, we have:
\begin{align}
\nonumber 
&\frac{1}{1+\theta L} \lone{\rho^\beta(x)\nabla u(x)}
 \\
\nonumber
& < \int_{\bbR^d} \int_{B(0,1)} \rho^\beta ( x - \theta \rho(x) y)
  |\nabla u(x-\theta \rho(x)y)| \phi(y)\,dy\,dx
 \\ 
\nonumber
& < \frac{1}{1-\theta L} \lone{ \rho^\beta(x)\nabla u(x)}
\end{align}
\end{corollary}
\begin{proof} (of Corollary~\ref{cor:lonetheta})
    Apply Lemma~\ref{lem:lonetheta} with $\delta(x) = \theta
    \rho(x)$, and $f(x) = \rho^\beta(x) \nabla u(x)$.
\end{proof}
This corollary will be used to bound the numerator of our
Rayleigh quotient.
Note that the expression
\[ \int_{\bbR^d} \int_{B(0,1)} \rho^\beta ( x - \theta \rho(x) y)
  |\nabla u(x-\theta \rho(x)y)| \phi(y) dy dx
  \]
is close to $\int_{\bbR^d} \rho^\beta(x) \nabla |u_\theta(x)| dx$
when $\theta \leq \frac{1}{2L}$. This is the guiding intuition behind how Corollary~\ref{cor:lonetheta} and
Lemma~\ref{lem:lonetheta} will be used, and will be formalized
later in our
proof of Theorem~\ref{thm:buser_n}.

We present another simple corollary whose proof is equally
straightforward. This corollary will be used to bound the
denominator, and is a small generalization of
Corollary~\ref{cor:lonetheta}. We write down both corollaries
anyhow, since this will make it easier to interpret our bounds on
the Rayleigh quotient.
\begin{corollary}\label{cor:lone-t-theta}
For any Lipschitz continuous function $\rho: \bbR^d \rightarrow \bbR^{\geq 0}$
with Lipschitz constant $L$, any $0 < t < 1$, and any $\theta$ with $0 < \theta L < 1$, we have:
\begin{align}
\nonumber 
&\frac{1}{1+\theta L} \lone{\rho^\beta(x)\nabla u(x)}
 \\
\nonumber
& < \int_{\bbR^d} \int_{B(0,1)} \rho^\beta ( x - \theta t \rho(x) y)
  |\nabla u(x-\theta t \rho(x)y)| \phi(y)\,dy\,dx
 \\ 
\nonumber
& < \frac{1}{1-\theta L} \lone{ \rho^\beta(x)\nabla u(x)}
\end{align}
\end{corollary}
\begin{proof} (of Corollary~\ref{cor:lone-t-theta})
    Apply Lemma~\ref{lem:lonetheta} with
    $\delta(x) = \theta t \rho(x)$, and $f(x) = \rho^\beta(x) \nabla u(x)$.
\end{proof}

Now we are ready to prove our main Theorem, which is the Buser
inequality for probability densities stated in Theorem~\ref{thm:buser_n}.
\begin{proof} (of Theorem \ref{thm:buser_n})

Fix $A \subset \RR^d$ with $|A|_\alpha \leq |1|_\alpha / 2$ and let
$u(x) = \chi_A(x)$ be the characteristic function of $A$. Setting 
$\ubar$ to be the weighted average of $u$,
\[
\ubar 
= \frac{\int \rho^\alpha u}{\int \rho^\alpha}
= \frac{\int_A \rho^\alpha}{\int \rho^\alpha}
= \frac{|A|_\alpha}{|1|_\alpha} \in [0,1/2],
\qquad \text{ then } \qquad
\int \rho^\alpha (u-\ubar) = 0,
\]
and
  \begin{equation}\label{eqn:blah}
\lone{\rho^\alpha(u-\ubar)} 
= \int \rho^\alpha |u-\ubar| = 2 |A|_\alpha (1-\ubar)
= 2 \int \rho^\alpha |u-\ubar|^2.
  \end{equation}

Since $|A|_\alpha = \lonea{u}$ and $1-\ubar \in [0,1/2]$ it follows that

\begin{equation} \label{eqn:PhiA}
  (1/2) \frac{\lone{\rho^\beta \nabla u}}{\lone{\rho^\alpha(u-\ubar)}} 
  \leq \Phi(A) = \frac{\lone{\rho^\beta \nabla u}}{\lone{\rho^\alpha u}}
  \leq \frac{\lone{\rho^\beta \nabla u}}{\lone{\rho^\alpha(u-\ubar)}}.
\end{equation}

In the calculations below we omit the limiting argument with smooth
approximations of $u$ outlined at the beginning of this section which
justify formula involving $\nabla u$, and for readability frequently
write $\rho$ and $\nabla \rho$ for $\rho(x)$ and $\nabla \rho(x)$.

Next, let $u_\theta$ be the mollification of (an extension of) $u$
given by equation \eqnref{:utheta}. Then $u_\theta(x)$ is a local
average average of $u$ so $u_\theta(x) \geq 0$, $\linf{u_\theta} \leq
1$ and $\linf{u-u_\theta} \leq 1$.  Letting $L$ denote the Lipschitz
constant of $\rho$, the parameter $\theta$ will to be chosen 
less than $1/(2L)$ so that that Lemma \ref{lem:lonetheta} is applicable
with constant $c = 1/2$.

The remainder of the proof constructs an upper bound on the numerator
  $\int_{\bbR^d} \rho^\gamma |\nabla u_\theta|^2$ of the Raleigh quotient
for $u_\theta - \ubar_\theta$ by $\lone{\rho^\beta \nabla u}$ and to
  lower bound the denominator $\int_{\bbR^d} \rho^\alpha (u_\theta -
\ubar_\theta)^2$ by $\lone{\rho^\alpha (u-\ubar)}$. The conclusion
of the theorem then follows from equation \eqnref{:PhiA}.

\subsection{Upper Bounding the Numerator}
To bound the $L^2$ norm in
  the numerator of the Raleigh quotient by the $L^1$ norm in the
  numerator of the expression for $\Phi(A)$ it is necessary to obtain
  uniform bound on $\rho(x) \nabla u_\theta(x)$. 
  
  \begin{lemma} \label{lem:rep1}
  Let $u$ be any function, and let $u_{\theta}$ be defined as in
  Equation~\ref{eqn:utheta}. Let $\rho: \Re^d \to \Re_{\geq 0}$ be an
  $L$-Lipschitz function.
  \begin{align}
  \linf{\rho (x) \nabla u_\theta(x)}
    \leq \linf{u} \frac{d(2+3L)}{\theta}
  \end{align}
  \end{lemma}

  \begin{proof}
  In order to prove this lemma, we first need to get a handle on
  $\nabla u_\theta(x)$, which is the gradient of $u$ after
  mollification by $\theta$.

  We take the
  the second representation of $u_\theta$ in equation \eqnref{:utheta}
  to get
  \begin{align}
  \nabla u_\theta(x) 
  = \int_{\bbR^d} u(z) \left\{
    \frac{-d}{\theta \rho(x)} \phi_{\theta \rho}(x-z) \nabla \rho
    + \frac{1}{(\theta \rho(x))^{d+1}} 
    \left( I 
      + \nabla \rho(x) \otimes \frac{x-z}{\theta \rho(x)} \right)
    \nabla \phi \left(\frac{x-z}{\theta \rho(x)} \right) \right\}
    \, dz,
  \end{align}
  which is a consequence of the multivariable chain rule. Here,
  $v \otimes u$ refers to the outer product of $v$ and $u$.

  Multiplying by $\rho$ gives:
  \begin{align}
  \rho (x) \nabla u_\theta(x) 
  = \int_{\bbR^d} u(z) \left\{
    \frac{-d}{\theta} \phi_{\theta \rho(x)}(x-z) \nabla
    \rho(x)
    + \frac{1}{(\theta^{d+1} \rho(x)^{d})} 
    \left( I 
      + \nabla \rho(x) \otimes \frac{x-z}{\theta \rho(x)} \right)
    \nabla \phi \left(\frac{x-z}{\theta \rho(x)} \right) \right\} \, dz.
  \end{align}
  Now, we can bound the above equation by carefully bounding each
  part. We note:
  \begin{align} \label{eq:Cphi-calculation}
    & \int_{\bbR^d} \frac{1}{(\theta^{d+1} \rho(x)^d)} \nabla
    \phi\left(\frac{x-z}{\theta \rho(x)} \right) dz
    \\
    & = \int_{\bbR^d} \frac{1}{(\theta^{d+1} \rho(x)^d)} \nabla
    \phi\left(\frac{-z}{\theta \rho(x)} \right) dz
    \\
     & = \frac{1}{\theta} \int_{\bbR^d} \nabla \phi(-y) dy
  \end{align}
  where the last step follows by a simple change of variable.
  Here, we note that $\nabla \phi(y)$ is a vector, and the
  integral is over $\bbR^d$, which is how we eliminated
  $\frac{1}{(\theta \rho(x))^d}$ from the expression.

  Next, we examine the term: 
  \begin{align}
    I + \nabla \rho(x) \otimes \frac{x-z}{\theta \rho(x)}
  \end{align}
  Here, we aim to bound the operator norm of this matrix. Here,
  we note that 
  \[ |x - z| \leq \theta \rho(x) \]
  when 
  \[
    \nabla \phi\left(\frac{x-z}{\theta \rho(x)}\right) \not= 0
  \]
  and thus, when the latter equation holds, we can say:

  \[
    \left| \frac{x-z}{\theta \rho(x)}\right| < 1.
  \]
  Since $|\nabla \rho(x) < L|$, we now have:
  \begin{align}\label{eqn:num-matrix-norm-bound}
    |I + \nabla \rho(x) \otimes \frac{x-z}{\theta \rho(x)}|_2 <
    3/2
  \end{align}
  Combining Equation~\ref{eqn:num-matrix-norm-bound} Equation~\ref{eq:Cphi-calculation} to show:
  \begin{align}
    & \left | \int_{\bbR^d}
    \frac{1}{(\theta^{d+1} \phi(x)^{d})} \left( I + \nabla \rho(x) \otimes \frac{x-z}{\theta \rho(x)}
    \right) \nabla \phi \left( \frac{x-z}{\theta \rho(x)}
    \right) dz \right| 
    \\
    & \leq \frac{(1 + L)}{\theta} \int_{\bbR^d} |\nabla\phi(y) | dy,
  \end{align}
  where $L = \linf{\nabla \rho}$ is the Lipschitz constant for $\rho$. 
  We note that Section~\ref{sec:dimension-dep} shows that 
  \begin{align}
    \int_{\bbR^d} | \nabla \phi(y) dy| \leq 2d.
  \end{align}
  and therefore:
  \begin{align}
    & \left | \int_{\bbR^d}
    \frac{1}{(\theta^{d+1} \phi(x)^{d})} \left( I + \nabla \rho(x) \otimes \frac{x-z}{\theta \rho(x)}
    \right) \nabla \phi \left( \frac{x-z}{\theta \rho(x)}
    \right) dz \right| 
    \\ \label{eq:rho-grad-utheta-2}
    & \leq \frac{2d(1+L)}{\theta}
  \end{align}

  Now we turn our attention to the first term, which is:
  \begin{align}
    \int_{\bbR^d} \frac{-d}{\theta} \phi_{\theta \rho(x)}(x-z)
    \nabla\rho(x) dz
  \end{align}
  We note that 
  \[ 
    \int_{\bbR^d}\left| \phi_{\theta \rho(x)}(x-z)\right| dz
    = 1
  \]
  by our definition of $\phi$ (which was defined when we defined
      $u_\theta$). Combining this
  with $|\nabla \rho(x)| < L$, we get:
  \begin{align}
    & \int_{\bbR^d} \left | \frac{-d}{\theta} \phi_{\theta \rho(x)}(x-z)
    \nabla\rho(x) \right| dz
    \\ \label{eq:rho-grad-utheta-1}
    & < \frac{dL}{\theta}
  \end{align}
  Therefore, 
  \begin{align}
  &  \nonumber \left| \int_{\bbR^d} 
    \frac{-d}{\theta \rho(x)} \phi_{\theta \rho}(x-z) \nabla \rho
    + \frac{1}{(\theta \rho(x))^{d+1}} 
    \left( I 
      + \nabla \rho(x) \otimes \frac{x-z}{\theta \rho(x)} \right)
    \nabla \phi \left(\frac{x-z}{\theta \rho(x)} \right) \, dz
    \right|
    \\
    & \nonumber \leq \frac{d}{\theta}( L + 2(1+L)) 
    \\
  & = \frac{d(2+3L)}{\theta}.
  \label{eqn:}
  \end{align}
  where the first inequality comes from combining
  Equations~\ref{eq:rho-grad-utheta-2}
  and~\ref{eq:rho-grad-utheta-1}.

  This allows us to bound $\linf{\rho(x) \nabla u_{\theta}(x)}$:
  \begin{align}
  & \linf{\rho (x) \nabla u_\theta(x)}
  \\
  & = \linf{ \left| \int_{\bbR^d} u(z) \left\{
    \frac{-d}{\theta} \phi_{\theta \rho(x)}(x-z) \nabla
    \rho(x)
    + \frac{1}{(\theta^{d+1} \rho(x)^{d})} 
    \left( I 
      + \nabla \rho(x) \otimes \frac{x-z}{\theta \rho(x)} \right)
    \nabla \phi \left(\frac{x-z}{\theta \rho(x)} \right) \right\}
  \, dz \right| }
  \\
  & \leq \linf{u} \linf{ \int_{\bbR^d} \left| 
    \frac{-d}{\theta} \phi_{\theta \rho(x)}(x-z) \nabla
    \rho(x)
    + \frac{1}{(\theta^{d+1} \rho(x)^{d})} 
    \left( I 
      + \nabla \rho(x) \otimes \frac{x-z}{\theta \rho(x)} \right)
    \nabla \phi \left(\frac{x-z}{\theta \rho(x)} \right)
  \right| \, dz}
  \\
  & \leq \linf{u} \frac{d(2+3L)}{\theta}
  \end{align}
  where we make use of the fact that
    $\linf{ab} < \linf{a}\lone{b}.$
  This completes our proof.
  \end{proof}

  Next, we want an $L_1$ bound on $\rho^{\beta}(x) \nabla
  u_{\theta}(x)$.

  \begin{lemma} \label{lem:rep2}
  Let $u$ be any function, and let $u_{\theta}$ be defined as in
  Equation~\ref{eqn:utheta}. Let $\rho:\Re^d \to \Re_{\geq 0}$ be an
  $L$-Lipschitz function, and
  let $\theta L < 1/2$.

  Then:
  \begin{align}
      \lone{\rho^\beta(x) \nabla u_{\theta}(x)} \leq C_{\beta}
      \lone{\rho^\beta(x) \nabla u(x)}
  \end{align}
  \end{lemma}
  \begin{proof}
  First, we take the gradient 
  first representation of $u_\theta$
  in equation \eqnref{:utheta}. Using the chain rule gives us an
  alternate form for $\nabla u_\theta(x)$:

  \begin{align}
  \nabla u_\theta(x) 
  = \int_{\Re^d} 
  (I - \theta \nabla \rho \otimes y) \nabla u(x-\theta \rho y) \phi(y) \, dy,
  \end{align}
  so
  \begin{align}\label{eqn:rhoNablaUtheta}
  \rho^\beta(x) \nabla u_\theta(x) 
  = \int_{\Re^d} 
  (I - \theta \nabla \rho \otimes y)
  \frac{\rho^\beta(x)}{\rho^\beta(x-\theta \rho y)}
  \rho^\beta(x-\theta \rho y) \nabla u(x-\theta \rho y) \phi(y) \, dy.
  \end{align}

  The ratio in the integrand is bounded using the Lipschitz assumption
  on $\rho$ (and $|y| \leq 1$),
  \begin{equation} \label{eqn:rhoRatio}
    \frac{\rho(x)}{\rho(x-\theta \rho y)}
    \leq \frac{\rho(x)}{\rho(x) - L \theta \rho(x)}
    = \frac{1}{1 - L \theta} \leq 2,
    \qquad \text{ when } \theta < 1 / (2L).
  \end{equation}
  Note that 
  \begin{align} \label{eqn:matrix-norm} 
  \left\| I - \theta \nabla \rho \otimes y \right\|_2 \leq 3/2
  \end{align}
  where $\|M\|_2$ represents the $\ell^2$ matrix norm of $M$. This is
  because $|\nabla \rho(x) | \leq L$, and
  $\theta L < 1/2$, and $|y| \leq
  1$ every time $\phi(y) \not= 0$, and thus 
  \[ 
   \frac{I}{2} \preceq  I - \theta \nabla \rho \otimes y
   \preceq \frac{3I}{2}.
    \]
  
 Therefore, we can now apply Corollary~\ref{cor:lonetheta}
 to Equation~\eqnref{:rhoNablaUtheta} to show:
  \begin{align}
  \nonumber 
  & \lone{\rho^\beta(x) \nabla u_\theta(x)} 
  \\  \nonumber
  & \leq \|I - \theta \nabla \rho(x) \otimes y\|_2
  \cdot \max_x\left(\frac{\rho(x)}{\rho(x - \theta \rho
        y)}\right) \cdot
  \int_{\bbR^d} \left| \int_{\bbR^d} \rho^{\beta}(x - \theta \rho y) \nabla u(x - \theta \rho y)
    \phi(y) dy \right|
  \\ \nonumber 
  & \leq 3 \cdot 2^{\beta - 1} \int_{\bbR^d} \rho^\beta(x -
      \theta \rho(x) y) \nabla u (x - \theta \rho (x) y 
  \\ \nonumber
  & 
  \leq 3 \cdot 2^\beta \lone{\rho^\beta \nabla u},
  \qquad \text{ when } \theta < 1 / (2L).
  \end{align}
  here, the first inequality comes from the equation $\lone{abc}
  \leq \linf{a}\linf{b}\lone{c}$, the second inequality comes
  from Equations~\ref{eqn:rhoRatio} and~\ref{eqn:matrix-norm}, and the
  third inequality comes
  from Corollary~\ref{cor:lonetheta} assuming $\theta L \leq
  1/2$. 
  \end{proof}

  \begin{lemma}\label{lem:num}
  For any $L$-Lipschitz distribution $\rho$, any function $u$, and any $\theta$ such that $\theta L <
  1/2$:

  \begin{align}
  \int_{\bbR^d} \rho^\gamma |\nabla u_\theta|^2
  \leq  C_\beta \linf{\rho^{\gamma-\beta-1}} \frac{d(2+3L)}{\theta} 
  \linf{u} \lone{\rho^\beta \nabla u},
  \end{align}
  \end{lemma}

  \begin{proof}
  Combining the two estimates from Lemma~\ref{lem:rep1}
  and~\ref{lem:rep2} gives an upper bound for the Raleigh
  quotient 
  \begin{align}
  \int_{\bbR^d} \rho^\gamma |\nabla u_\theta|^2
  = \int_{\bbR^d} \rho^{\gamma-\beta-1} \,
  \rho |\nabla u_\theta| \, \rho^\beta |\nabla u_\theta|
  \leq 3 \cdot 2^{\beta + 1} \linf{\rho^{\gamma-\beta-1}} \frac{d(2+3L)}{\theta} 
  \linf{u} \lone{\rho^\beta \nabla u},
  \end{align}
  \end{proof}

  We note that in the case where $\gamma = \beta + 1$, and if $u$ is a step
  function, the expression would
  simplify to:
  \[ \int_{\bbR^d}
  |\rho^{\gamma} |\nabla u_\theta|^2
  \leq 3 \cdot 2^{\beta + 1} \frac{d(2+3L)}{\theta} 
  \linf{u} \lone{\rho^\beta \nabla u},
  \]

\subsection{Lower Bound on the Demoninator}
Let $\ubar$ and
  $\ubar_\theta$ be the $\rho^\alpha$--weighted averages of $u$ and
  $u_\theta$ and let $\ltwoa{.}$ denote the $L^2$ space with this
  weight. 
  Our core lemma is a bound on $\ltwoa{u_\theta-\ubar_\theta}$ in
  terms of $l_1$ and weighted $l_1$ norms of $\nabla u$ and $u -
  \ubar$ respectively.
  \begin{lemma}\label{lem:denom}
  Let $\rho$ be an $L$-Lipschitz function $\rho: \Re^d \to
  \Re_{\geq 0}$, and let $\theta$ be such
  that $\theta L < 1/2$.
  Let $u$ be an indicator function of a set $A$ with finite
  $\beta$-perimeter. Let $\ubar$ be defined as $\ubar(x):=u(x) - \int u(y) dy$
    $u_\theta$ be defined as in Equation~\ref{eqn:utheta}, and
    $\ubar_\theta$ be defined as $\ubar_\theta(x) := u_\theta(x)
    - \int
    u_\theta(y) dy$.
    Then:
  \begin{align}
  \ltwoa{u_\theta - \ubar_\theta}^2
  \geq (1/4) \lonea{u-\ubar} 
  - C(\beta) \theta \linf{\rho^{\alpha+1-\beta}} \lone{\rho^\beta \nabla u},
  \qquad \text{ when } \theta < 1 / (2L).
  \end{align}
  \end{lemma}

  Note that when $\alpha + 1 = \beta$, as is true when $(\alpha,
      \beta, \gamma) = (1,2,3)$, the inequality in
  Lemma~\ref{lem:denom} becomes:

  \[
  \ltwoa{u_\theta - \ubar_\theta}^2
  \geq (1/4) \lonea{u-\ubar} 
  - C(\beta) \theta \lone{\rho^\beta \nabla u},
  \qquad \text{ when } \theta < 1 / (2L).
  \]
  The estimate in Lemma~\ref{lem:denom} will be combined with the
  estimate in Lemma~\ref{lem:num} to prove
  Theorem~\ref{thm:buser_n}
  in Section~\ref{sec:rayleigh-bound}.

  \begin{proof} The key to this proof is to upper bound the
  quantity $\ltwoa{u_\theta - \ubar_\theta}$ with the expression
  appearing in Corollary~\ref{cor:lone-t-theta}. We will do so by
  a series of inequalities, application of the fundamental
  theorem of calculus, and more.

  Using the property that subtracting the average from a
  function reduces the $L^2$ norm it follows that
  \begin{align}
  \nonumber
  & \ltwoa{u_\theta - \ubar_\theta}
  \\
    \nonumber
  & \geq \ltwoa{u-\ubar} - \ltwoa{u_\theta - u - (\ubar_\theta-\ubar)}
  \\
    \nonumber
  &\geq \ltwoa{u-\ubar} - \ltwoa{u_\theta - u}.
  \end{align}
  If $a \geq b-c$ then $a^2 \geq b^2/2 - c^2$, so a lower bound for
  the denominator of the Raleigh quotient
  \begin{align} \label{eqn:utmu}
    & \ltwoa{u_\theta - \ubar_\theta}^2
    \\
    \nonumber
    &\geq (1/2) \ltwoa{u-\ubar}^2 - \ltwoa{u_\theta - u}^2
    \\
    \nonumber
    & \geq (1/4) \lonea{u-\ubar} - \lonea{u_\theta - u},
  \end{align}
  where the identity $\ltwoa{u-\ubar}^2 = \lonea{u-\ubar}/2$ from
  Equation~\ref{eqn:blah}, and the bound
  $\linf{u_\theta - u} \leq 1$, were used in the last step.

  It remains to estimate the difference $\lonea{u_\theta - u}$. To
  do this, we use the multivariable fundamental
  theorem of calculus to write
  \begin{eqnarray*}
    u_\theta(x) - u(x)
    &=& \int (u(x - \theta \rho y) - u(x)) \phi(y) \, dy \\
    &=& \int \! \int_0^1
    -\theta \rho(x) \nabla u(x - t \theta \rho(x) y).y \phi(y) \, dt \, dy \\
    &=& \int \! \int_0^1
    \frac{-\theta \rho(x)}{\rho^\beta(x-t\theta \rho(x) y)} 
    \rho^\beta(x - t\theta \rho(x) y) \nabla u(x - t \theta
        \rho(x) y).y \phi(y) \, dt \, dy,
  \end{eqnarray*}
  where the first and second equalities came from application of
  the multivariable fundamental theorem of calculus, and the last
  equation is straightforward. This tells us that:
  \begin{align}
  \notag
  & \rho^{\alpha}(x) ( u_\theta \rho(x) - u(x) )
  \\
  \notag
  & = \int \! \int_0^1
  \frac{-\theta \rho^{\alpha+1}(x)}{\rho^\beta(x-t\theta \rho y)} 
  \rho^\beta(x - t\theta \rho(x) y) 
  \nabla u(x - t \theta \rho(x) y).y \phi(y) \, dt \, dy.
  \\
    \label{eqn:utheta-l1-alpha}
  &= \int \! \int_0^1
  \frac{\rho^\beta(x)}{\rho^\beta (x- t \theta \rho(x)y)}
  \frac{-\theta \rho^{\alpha+1}(x)}{\rho^\beta(x)}
  \rho^{\beta}(x - t \theta \rho(x) y )
  \nabla u(x - t \theta \rho(x) y).y \phi(y) \, dt \, dy.
  \end{align}
  Equation \eqnref{:rhoRatio} bounds the ratio $\rho(x)/ \rho(x -
    t\theta \rho(x) y)$ as less than $2$ when $\theta L < 1/2$,
  so Equation~\eqnref{:utheta-l1-alpha} is always less than or
  equal to:
  \begin{align}
  \int \! \int_0^1
  2^{\beta}
  \frac{-\theta \rho^{\alpha+1}(x)}{\rho^\beta(x)}
  \rho^{\beta}(x - t \theta \rho(x) y )
  \nabla u(x - t \theta \rho(x) y).y \phi(y) \, dy \, dt.
  \end{align}

  An application of Corollary \ref{cor:lone-t-theta} then shows
  \begin{align}
  & \int \! \int_0^1
  2^{\beta}
  \frac{-\theta \rho^{\alpha+1}(x)}{\rho^\beta(x)}
  \rho^{\beta}(x - t \theta \rho(x) y )
  \nabla u(x - t \theta \rho(x) y).y \phi(y) \, dy \, dt.
  \\
  & \leq \int \! \int_0^1
  2^{\beta}
  \frac{-\theta \rho^{\alpha+1}(x)}{\rho^\beta(x)}
  \rho^{\beta}(x - t \theta \rho(x) y ) \ 
  \left|\nabla u(x - t \theta \rho(x) y) \phi(y)\right| \, dy \, dt.
  \\
  &
  \leq 2^{\beta+1} \linf{\rho^{\alpha+1-\beta}} \theta
  \int \! \int_0^1
  \rho^{\beta}(x - t \theta \rho(x) y ) \ 
  \left|\nabla u(x - t \theta \rho(x) y) \phi(y)\right| \, dy \, dt.
  \qquad \text{ when } \theta < 1 / (2L).
  \\
  &
  \leq 2^{\beta+1} \linf{\rho^{\alpha+1-\beta}} \theta \lone{\rho^\beta \nabla u}
  \end{align}
  where the last inequality follows from
  Corollary~\ref{cor:lone-t-theta}.

  Using this estimate in \eqnref{:utmu} gives a lower bound on the
  denominator of the Raleigh quotient,
  \begin{align}
  \ltwoa{u_\theta - \ubar_\theta}^2
  \geq (1/4) \lonea{u-\ubar} 
  - 2^{\beta+1} \theta \linf{\rho^{\alpha+1-\beta}} \lone{\rho^\beta \nabla u},
  \qquad \text{ when } \theta < 1 / (2L).
  \end{align}
  as desired.
  \end{proof}

  \subsection{Bounding the Rayleigh
    Quotient (Proof of Theorem~\ref{thm:buser_n})}\label{sec:rayleigh-bound}
  Combining Lemmas~\ref{lem:num} and Lemmas~\ref{lem:denom}
  provides an upper bound for the Rayleigh quotient of $u_\theta - \ubar_\theta$,
  \begin{eqnarray*}
    \lambda_2 
    &\leq& \frac{\int_{\bbR^d} \rho^\gamma |\nabla u_\theta|^2}
    {\int_{\bbR^d} \rho^\alpha (u_\theta - \ubar_\theta)^2} \\
    &\leq&  \frac{d \cdot 3 \cdot 2^{\beta}}{\theta}
    \frac{\linf{\rho^{\gamma-\beta-1}} (2+3L) \lone{\rho^\beta \nabla u}}
    {\lonea{u-\ubar} 
      - 2^{\beta+1} \theta \linf{\rho^{\alpha+1-\beta}} \lone{\rho^\beta \nabla u}} \\
    &\leq&  \frac{d \cdot 3 \cdot 2^{\beta}}{\theta}
    \frac{\linf{\rho^{\gamma-\beta-1}} (2+3L)}
    {1 - 2^{\beta + 1} \theta \linf{\rho^{\alpha+1-\beta}} \Phi(A)} \Phi(A).
  \end{eqnarray*}
  Selecting $\theta = (1/2)
  \min\left(1/\left(2^{\beta+1}\linf{\rho^{\alpha+1-\beta}}
        \Phi(A)\right), 1/L
  \right)$ shows
  \[
  \lambda_2 \leq 2 d \cdot 3 \cdot 2^{\beta}
\linf{\rho^{\gamma-\beta-1}}(2+3L) 
  \max\left(L \Phi(A), 2^{\beta+1}\linf{\rho^{\alpha+1-\beta}} \Phi(A)^2 \vph\right).
  \]
  When $\gamma = (1,2,3)$, this simplifies into:
  \[
  \lambda_2 \leq 12(2+3L) d
  \max\left(L \Phi(A), 8 \Phi(A)^2 \vph\right).
  \]
  We note that, via the work shown in
  Section~\ref{sec:scaling},  we can strengthen our inequality
  to:
  \[
  \lambda_2 \leq 24 d
  \max\left(L \Phi(A), 8 \Phi(A)^2 \vph\right).
  \]

  \qedhere

\subsection{Gradient of Mollifier}\label{sec:dimension-dep}
Let $\phi$ be a standard mollifier i.e. $\phi\in C_c^\infty(\RR^d)$
is a function from $\RR^d\to [0,\infty)$ satisfying $\int_{\RR^d}
  \phi\,dx = 1$ and $\supp(\phi)\subseteq B(0,1)$.  We will define
  $\phi$ by its profile. Namely, let $\phihat(r):[0,\infty)
    \rightarrow [0,1]$ be a fixed monotone decreasing profile with
    $\phihat(0)=1$, $0 < \phihat(r) < 1$ for $0 < r < 1$, and
    $\phihat(r) = 0$ for $r \geq 1$. Then define $\phi:\Re^d
    \rightarrow \Re$ by $\phi(x) = c \phihat(|x|)$ with $c > 0$ chosen
    so that $\int_{\Re^d} \phi(x) \, dx = 1$; that is,

\[
1 = \int_{\Re^d} \phi(x) \, dx
= c |S^{d-1}| \int_0^1 \phihat(r) r^{d-1} \, dr
\qquad \Rightarrow \qquad
c = \frac{1}{|S^{d-1}| \int_0^1 \phihat(r) r^{d-1} \, dr},
\]

where $|S^{d-1}|$ is the $(d-1)$--area of the unit sphere in $\Re^d$.
We claim the $L_1$ norm of the gradient of $\nabla \phi(x)$ is linear in $d$.
\begin{lemma}\label{lem:molli}
  \[
  \int_{\Re^d} |\nabla \phi(x)| \, dx   
\leq (d-1) \left( \frac{d 2^d}{\phihat(1/2)} \right)^{1/(d-1)}
\stackrel{d \rightarrow \infty}{\longrightarrow} 2(d-1).
\]
For the classic mollifier $\phihat(r)=\exp(-1/(1-r^2))$ we get
  \[
  \int_{\Re^d} |\nabla \phi(x)| \, dx  \leq 2d.
\]
\end{lemma}

From the formula $\nabla \phi(x) = c \phihat'(|x|) (x/|x|)$ we compute
\begin{eqnarray*}
\int_{\Re^d} |\nabla \phi(x)| \, dx
&=& c |S^{d-1}| \int_0^1 |\phihat'(r)| r^{d-1} \, dr \\
&=& c |S^{d-1}| \int_0^1 -\phihat'(r) r^{d-1} \, dr \\
&=& c |S^{d-1}| \int_0^1 \phihat(r) (d-1) r^{d-2} \, dr \\
&=& (d-1) \frac{\int_0^1 \phihat(r) r^{d-2} \, dr}
{\int_0^1 \phihat(r) r^{d-1} \, dr}.
\end{eqnarray*}
To estimate the numerator use Holder's inequality: for
$1 \leq s, s' \leq \infty$ with $1/s + 1/s' = 1$ 
$$
\int f g 
\leq \left(\int |f|^s \right)^{1/s} \left(\int |g|^{s'} \right)^{1/s'}.
$$
Set $s = (d-1)/(d-2)$ and $s' = d-1$ to get
$$
\int_0^1 \phihat(r) r^{d-2} \, dr
= \int_0^1 \phihat(r)^{1/s} r^{d-2} \times \phihat(r)^{1/s'}  \, dr
\leq \left(\int_0^1 \phihat(r) r^{d-1} \, dr \right)^{1/s}
\left(\int_0^1 \phihat(r) \, dr \right)^{1/s'}.
$$
It follows that
\begin{equation}\label{eq:molli0}
\int_{\Re^d} |\nabla \phi(x)| \, dx
\leq (d-1) \left( \frac{\int_0^1 \phihat(r) \, dr}
{\int_0^1 \phihat(r) r^{d-1} \, dr} \right)^{1/(d-1)}.
\end{equation}
Since $0 \leq \phihat(r) \leq 1$ we can bound the numerator
by $1$, and since $\phihat(r)$ is monotone decreasing we have
$\phihat(r) \geq \phihat(1/2)$ on $(0,1/2)$, so 
\begin{equation}\label{eq:molli}
\int_{\Re^d} |\nabla \phi(x)| \, dx
\leq (d-1)
\left( \frac{1}{\phihat(1/2) \int_0^{1/2} r^{d-1} \, dr} \right)^{1/(d-1)} 
\leq (d-1) \left( \frac{d 2^d}{\phihat(1/2)} \right)^{1/(d-1)}
\stackrel{d \rightarrow \infty}{\longrightarrow} 2(d-1).
\end{equation}

It will be convenient to write equation~\ref{eq:molli} as a simple
inequality. Observer that
\[
 \left( \frac{d 2^d}{\phihat(1/2)} \right)^{1/(d-1)}
\]
is monotone decreasing. We now pick the classic
$\phihat(r) = \exp(-1/(1-r^2))$ we have $\phihat(1/2) \geq 1/4$ and if
$d \geq 5$ the right hand side of equation (\ref{eq:molli} is bounded
by $2d$. If $d < 5$ explicit computations of the integrals shows the
right hand side of equation (\ref{eq:molli0} is bounded by $2d$.

\end{proof}

\subsection{Scaling}\label{sec:scaling}

In this section we show that if one scales the density function $\rho$ then
the isoperimetric value $\Phi(A)$ and the Raleigh quotient $R(u)$
scale nicely. More formally Let $A \subset \Omega \subseteq
\Re^d$, $\rho$ a density function over a domain $\Omega$, and
$u$ an arbitrary differentiable  function over $\Omega$.

Consider the transformation
$\xhat = \ell x$ with $\ell > 0$ which maps $\Omega$ to the domain
$\Omegahat = \{\ell x \sst x \in \Omega\}$. Given $u:\Omega \rightarrow
\Re$, we define $\uhat: \Omegahat \rightarrow \Re$ by $\uhat(\xhat) =
u(x)$. We will future scale $\rho$ by $\alpha \rhohat(\xhat) = \ell \, \rho(x)$ where
$\alpha > 0$.

\begin{theorem}\label{thm:scaling}
  When scaling by $\alpha$ and $\ell$ then
  \[   \Phi(A) = \alpha \Phihat(\Ahat) \] and
  \[  R(u)  = \alpha^2 \Rhat(\uhat) \quad \text{and thus} \quad\lambda_2 = \alpha^2 \hat{\lambda}_2\].
\end{theorem}

We will use this scaling theorem to improve the bounds
of theorem~\ref{thm:buser_n}. 

That is, if we have a density function $\rho$ over
a domain $\Omega$ the isoperimetric number that the fundamental eigenvalue only
change as a function  of the scaling.  Thus the optimal cut and eigenvector are
unchanged by scaling up to the transformation.


If $u$ and $l$ are as defined above then we get the simple but basic identity. 
Suppose that $u: \R \rightarrow \R$ then: \[
\dbydp{u}{x} 
= \dbydp{\uhat}{\xhat} \dbydp{\xhat}{x}
= \dbydp{\uhat}{\xhat} \ell,
\qquad \text{ in general we get } \qquad
|\nabla u(x)| = \ell |\hat{\nabla} \uhat(\xhat)|.
\]

In the case of $\rho:\Omega \rightarrow (0,\infty)$, where
$\rhohat: \Omegahat \rightarrow (0,\infty)$ is defined by
$ \alpha \rhohat(\xhat) = \ell \, \rho(x)$ we get that
\[
|\nabla \rho(x)| = \alpha |\hat{\nabla} \rhohat(\xhat)|.
\]

It follows that $L_{\rhohat}$ and $L_\rho$, the Lipschitz constants
for $\rhohat$ and $\rho$,  satisfy $L_{\rhohat} = (1/\alpha) L_\rho$.

\begin{itemize}
\item
  Since $d\xhat = \ell^{d} \, dx$ we have
  \[
  \int_\Omega \rho \, dx = \frac{\alpha}{\ell^{d+1}} \int_{\Omegahat} \rhohat \, d\xhat, 
  \]

\item
  If $A \subset \Omega$ and $\Ahat = \ell A \subset \Omegahat$,
  let $f_A(x) = 1$ if $x \in A$ and zero otherwise, and similarly
  $f_{\Ahat} = 1$ if $\xhat \in \Ahat$ and zero otherwise. We next perform
  a set of standard integral calculations.
  \begin{align}
  \int_\Omega \rho^2 |\nabla f_A| \, dx &=
  \int_{\Omegahat} (\frac{\alpha}{\ell})^2 \rhohat^2 \ell |\hat{\nabla} f_{\Ahat}|
  \, \frac{1}{\ell^d} d\xhat \label{eq:subsitution}\\
  &= \frac{\alpha^2}{\ell^{d+1}} \int_{\Omegahat} \rhohat^2  |\hat{\nabla} f_{\Ahat}|
  \, d\xhat \label{eq:integral1}
  \end{align}

  Equation~\ref{eq:subsitution} follows by making the substitutions:
  \[ \rho(x) = (\frac{\alpha}{\ell}) \rhohat(\xhat) \quad
  |\nabla f_A| = \ell |\hat{\nabla} f_{\Ahat}| \quad
  dx =  \frac{1}{\ell^d} d\xhat
  \]
  Observing the $f_A(x) = f_{\Ahat}(\xhat)$ we get the following identity.
\begin{equation}\label{eq:integral2}
 \int_\Omega \rho f_A \, dx
 = \int_{\Omegahat} \frac{\alpha}{\ell}  \rhohat f_{\Ahat} \frac{1}{\ell^d} \, d\xhat
   = \frac{\alpha}{\ell^{d+1}} \int_{\Omegahat} \rhohat f_{\uhat} \, d\xhat
\end{equation}

Combining equation~\ref{eq:integral1} and equation~\ref{eq:integral2}
we get that:

\begin{equation}
  \Phi(A) = \alpha \Phihat(\Ahat)
\end{equation}

\item
  We next do a similar calculation for the Raleigh quotient.
  If $u:\Omega \rightarrow \Re$ and $\uhat(\xhat) = u(x)$, 
  the Raleigh quotients can be computed as follows,
  \[
  \int_\Omega \rho^3 |\nabla u|^2 \, dx
  = \int_{\Omegahat} (\frac{\alpha}{\ell})^3 \rhohat^3
  \ell^2 |\hat{\nabla} \uhat|^2 \, \frac{1}{\ell^d} d\xhat
    = \frac{\alpha^3}{\ell^{d+1}} \int_{\Omegahat} \rhohat^3 |\hat{\nabla} \uhat|^2 \, d\xhat
 \]

\[
\int_\Omega \rho u^2 \, dx
= \int_{\Omegahat} \frac{\alpha}{\ell} \rhohat \uhat^2 \frac{1}{\ell^d} \, dx
= \frac{\alpha}{\ell^{d+1}} \int_{\Omegahat} \rhohat \uhat^2 \, dx
\]
Thus \[  R(u)  = \alpha^2 \Rhat(\uhat) \].
\end{itemize}

We next use our scaling result in the $(1,2,3)$ case
to our Buser-type bound, Theorem~\ref{thm:buser_n}.
Theorem~\ref{thm:buser_n} states that the following hold:

\begin{align}\label{eq:simple123}
  \lambda_2 \leq 24 d (1+L) 
  \max\left(L \Phi(A), 12 \Phi(A)^2 \vph\right).
\end{align}

We now make substitutions into equation~\ref{eq:simple123}
from Theorem~\ref{thm:scaling} and its proof for some parameter $\alpha$ to be determined.

\begin{align*}
  \lambda_2 = \alpha^2 \hat{\lambda}_2 &\leq  \alpha^2  24 d (1+\hat{L}) 
  \max\left(\hat{L} \Phihat(\Ahat), 12 \Phihat(\Ahat)^2 \vph\right) \\
  &=  \alpha^2  24 d (1+ (L/\alpha))
  \max\left((L/\alpha) (\Phi(A)/\alpha), 12 (\Phi(A)/\alpha)^2 \vph\right) \\
  &=  24 d (1+ (L/\alpha))
  \max\left(L \Phi(A), 12 \Phi(A)^2 \vph\right) \\
  &=  24 d \max\left(L \Phi(A), 12 \Phi(A)^2 \vph\right) \quad \because{\alpha \rightarrow \infty} \\
\end{align*}

Thus we get that $\lambda_2$ only depends linear in the dimension and  the Lipschitz constant:

\begin{corollary}\label{cor:strongBuser}
  \[
  \lambda_2 \leq  24 d \max\left(L \Phi, 12 \Phi^2 \vph\right)
  \]
\end{corollary}

\section{Cheeger Inequality for Probability Density Functions}\label{sec:cheeger}

In this section, we prove the Cheeger inequality from
Theorem~\ref{thm:Cheeger-Buser}. That is a weighted Cheeger inequality in higher dimensions.
This is the easier to prove than Buser's inequality, which
contrasts with what happens in the graph case (the graph Buser
    inequality is trivial).

For a simplified proof of the Cheeger inequality for distributions
in one-dimension, see Appendix~\ref{sec:one_dim}.

As we will see from simple counterexamples in
Section~\ref{sec:examples}, the Cheeger-direction does not
hold for all setting of $(\alpha,\beta,\gamma)$. The proof we give is
requires fewer assumptions than the Buser inequality for
probability densities. One, the Cheeger inequality is
independent of the Lipschitz constant of $\rho$ and two, the proof
also holds when $\rho$ is supported on a set
$\Omega \subset \mathbb{R}^d$.

The proof is almost identical to the proof in one dimension and only a
slight modification of standard proofs The only change in the proof is
replacing the change of variables formula with a co-area formula.  Let
$\rho:\Omega \to \RR_>$ be an Lipschitz density function that is
$(\alpha,\beta,\gamma)$-integrable over an open set $\Omega \subseteq
\RR^d$.  Note a stronger hypothesis  on $\Omega$ is that it is the support of
$\rho$ when  $\rho:\RR^d \to \RR_\leq$. 

\begin{theorem}
\label{thm:cheeger_n}
Let $\rho:{\Omega}\to\RR_{>0}$ be a Lipschitz function. Then,
\begin{align*}
\Phi^2 \leq 4 \norm{\rho^{\beta - \frac{\alpha+\gamma}{2}}}^2_\infty \lambda_2.
\end{align*}
In particular, when $(\alpha,\beta,\gamma) = (1,2,3)$ we have
\begin{align*}
\Phi^{2} \leq 4\lambda_2.
\end{align*}
\end{theorem}
Here, $\Phi$ is the optimal $(\alpha,\beta)$-sparsity of a cut
through $\rho$. We note that we can say something a little stronger:
\begin{theorem}
\label{thm:cheeger-sweep}
Let $\rho:{\Omega}\to\RR_{>0}$ be a Lipschitz function. 
Let $\Phi_{(\alpha,\beta,\gamma)}$ be the $(\alpha,\beta)$
sparsity of the $(\alpha, \gamma)$ spectral sweep cut. If $\alpha = \beta-1 = \gamma-2$,  then:
\begin{align*}
\Phi_{(\alpha, \beta, \gamma)}^2 \leq 4 \lambda_2
\end{align*}
\end{theorem}
\begin{proof} (of both theorems): 
Let $w\in W^{1,2}$, functions whose gradient is square integrable, nonzero with $\int_\Omega \rho^\alpha w\,dx = 0$. Let $v = w+a1$ where $a$ is chosen such that $\abs{\set{v<0}}_\alpha = \abs{\set{v>0}}$. Note that
\begin{align*}
R(w) &= \frac{\int_\Omega \rho^\gamma \abs{\grad w}^2\,dx}{\int_\Omega \rho^\alpha w^2\,dx}\\
&\geq \frac{\int_\Omega \rho^\gamma \abs{\grad w}^2\,dx}{\int_\Omega \rho^\alpha w^2\,dx+ a^2\abs{\Omega}_\alpha}\\
&= R(v).
\end{align*}
Without loss of generality, the function $u = \max(v,0)$ satisfies $R(u)\leq R(v)$.

Let $\Omega_0 = \set{v>0}$. Let $g = u^2$. Noting that $\grad g = 2u\grad u$ a.e., we can apply Cauchy-Schwarz to obtain
\begin{align*}
\int_{\Omega_0}\rho^\beta \abs{\grad g}\,dx &= 2\int_{\Omega_0}\rho^\beta \abs{u}\abs{\grad u}\,dx\\
&\leq 2 \sqrt{\int_{\Omega_0} \rho^{2\beta - \alpha} \abs{\grad u}^2\,dx}\sqrt{\int_{\Omega_0} \rho^{\alpha} u^2\,dx}\\
&\leq 2 \norm{\rho^{\beta - \frac{\alpha+\gamma}{2}}}_\infty \sqrt{\int_{\Omega_0} \rho^\gamma \abs{\grad u}^2\,dx}\sqrt{\int_{\Omega_0} \rho^{\alpha} u^2\,dx}.
\end{align*}
Then, dividing by $\int_{\Omega_0} \rho^\alpha g\,dx$, we have
\begin{align*}
\frac{\int_{\Omega_0} \rho^\beta \abs{\grad g}\,dx}{\int_{\Omega_0} \rho^\alpha g\,dx} &\leq2 \norm{\rho^{\beta - \frac{\alpha+\gamma}{2}}}_\infty \sqrt{R(w)}.
\end{align*}
Let $A_t = \set{g>t}$. Then, by the weighted co-area formula,
\begin{align*}
\int_{\Omega_0} \rho^\beta \abs{\grad g}\,dx = \int_0^\infty \abs{\boundary A_t}_\beta\,dt.
\end{align*}
Writing $g(x) = \int_0^{g(x)} 1 \,dt$ and applying Tonelli's theorem, we rewrite the denominator
\begin{align*}
\int_{\Omega_0} \rho^\alpha g\,dx = \int_{0}^\infty \abs{A_t}_\alpha\,dt.
\end{align*}
Thus, by averaging, there exists some $t^*$ such that
\begin{align*}
\Phi &\leq \Phi(A_{t^*})\\
&\leq \frac{\int_{\Omega_0} \rho^\beta \abs{\grad g}\,dx}{\int_{\Omega_0} \rho^\alpha g\,dx}\\
&\leq2 \norm{\rho^{\beta - \frac{\alpha+\gamma}{2}}}_\infty \sqrt{R(w)}.
\end{align*}
Optimizing over the set $\set{w\in W^{1,2}\smid w\neq 0,\,\int_\Omega \rho^\alpha w\,dx = 0}$ completes the proof.
\end{proof}

\section{Spectral Sweep Cuts have Provably Good Sparsity (proof of
    Theorem~\ref{thm:sweep-cut})}\label{sec:sweep_cut}

Theorem~\ref{thm:cheeger-sweep} tells us that 
  \[ \Phi_{(1,2, 3)}^2/4 \leq \lambda_2^{(1,3)} \] for all $1$-Lipschitz $\rho$ whose
  domain is
  on $\mathbb{R}^d$.  Here, $\phi_{(1,2,3)}$ is the $(1,2)$ sparsity of the
  $(1,3)$-spectral sweep cut, and $\lambda_2^{(1,3)}$ is the
  $(1,3)$-principal eigenvalue. 

  Next Theorem~\ref{thm:buser_n} tells us that
 \[ \lambda_2^{(1,3)} \leq O(d \Phi_{(1,2)}), \]
 where $\Phi_{(1,2)}$ is the minimal $(1,2)$-sparsity of any cut through
 $\rho$.

 Therefore, 
 \[ \Phi_{(1,2,3)}^2 \leq  \Phi_{(1,2)} \leq \Phi_{(1,2,3)}^2,\] 
 where $\Phi_{(1,2)}$ is the minimum $(1,2)$-sparsity of a cut through
 $\rho$, proving Theorem~\ref{thm:sweep-cut}.


\section{Cheeger-Buser inequality fails for Bad Settings of
$\alpha, \beta, \gamma$: Examples}
\label{sec:examples}


In this section, we will analyze some simple $1$-Lipschitz
density functions in $1$ dimension, and see the requirements we
need on $(\alpha, \beta, \gamma)$ in order to recover Cheeger and
Buser type inequalities. This will prove
Lemma~\ref{lem:converse}.

Specifically, we refer to an inequality of the form
\begin{align*}
  \Phi^2 <  O(\lambda_2)
\end{align*}
as a Cheeger-type inequality and an inequality of the form
\begin{align*}
  \lambda_2 < O( \max(\Phi,\Phi^2))
\end{align*}
as a Buser-type inequality. The presence of the $\Phi^2$ term in
the Buser-type inequality is necessary as $\Phi$ may be larger
than $1$; this contrasts the normalized graph case where $\Phi$
is always bounded above by $1$.

We consider two simple density functions: the first is the
function $\rho$ that's $1/n$ for a support of interval $[-n/2,
n/2]$. The distribution
is then made to be Lipschitz in the obvious way, by making a
Lipschitz drop-off at the ends of the interval, and the interval
is rescaled so that it is a probability density function. In this
analysis, we assume $\alpha, \beta, \gamma$ to be constant.
The examples are discontinuous across the boundaries, $-1,1$, but it straight forward to extend them to be
continuous on their boundary without appreciably changing the them as counter examples.

In this case, the cut size in the isoperimetric cut is
within a constant approximation of $1/n^{\beta}$, the mass term in the isoperimetric
cut is within a constant approximation of $n/n^\alpha$, so the
isoperimetric cut is approximately $n^{\alpha - \beta - 1}$. and the
Rayleigh quotient can be bounded above using the Hardy Muckenhoupt
inequality~\cite{muckenhoupt1972hardy, MillerHardy18}, which shows that it is a constant approximation of
$n^{\alpha - \gamma - 2}$. 

The Cheeger inequality would then say:

\[n^{2(\alpha - \beta - 1)} < O(n^{\alpha - \gamma - 2}) \]

or 

\begin{equation}
\label{eq:b_equals_avg_a_c}
  \frac{\alpha + \gamma }{2} \leq \beta
\end{equation} for any Cheeger
inequality to hold.

We now turn our attention to a scenario in which the Buser
direction will fail for improperly set $\alpha, \beta, \gamma$.

\subsection{Notation}
We will write $a\gtrsim b$ if $a\geq cb$ for some absolute constant $0<c<\infty$. Similarly define $a \lesssim b$. We will write $a\asymp b$ if both relations hold.

\subsection{A Lipschitz weight}
\label{subsec:lipschitz_example}
Consider the density function $\rho(x) = \abs{x} + \epsilon$ on
the domain $(-1,1)$, where $\epsilon\in(0,1/4)$, and $\rho(x) =
\max(0, 2+\epsilon-|x|)$ for all other $x$.

It is clear that
\begin{align*}
\Phi(\Omega) &=\Phi(0) \asymp \epsilon^\beta.
\end{align*}
Next, we apply the Hardy-Muckenhoupt inequality to estimate $\lambda_2$. We upper bound $\H$ as
\begin{align*}
\H &\leq\R(1)\M(0)\\
  &\asymp  \int_0^1 \frac{1}{(x+\epsilon)^{\gamma}}\,dx\\
& \lesssim \begin{cases}
  1 & \text{if } \gamma<1\\
  \ln\left(1/\epsilon\right)& \text{if } \gamma = 1\\
  O\epsilon^{1-\gamma}& \text{if } \gamma>1.
\end{cases}
\end{align*}
Thus, if we want a Buser-type inequality to hold, then $(\alpha,\beta,\gamma)$ needs to satisfy,
\begin{align*}
\begin{cases}
  1\lesssim \lambda_2 \lesssim \max(\Phi,\Phi^2)  \asymp \epsilon^\beta & \text{if } \gamma<1\\
  \frac{1}{\ln\left(1/\epsilon\right)}\lesssim \lambda_2 \lesssim \max(\Phi,\Phi^2)  \asymp \epsilon^\beta& \text{if } \gamma = 1\\
  \epsilon^{\gamma-1}\lesssim \lambda_2 \lesssim \max(\Phi,\Phi^2)  \asymp \epsilon^\beta& \text{if } \gamma>1.
\end{cases}
\end{align*}
By letting $\epsilon$ go to zero, it is clear that $\gamma-1\geq \beta$.

By combining the condition \eqref{eq:b_equals_avg_a_c} required for a Cheeger-type inequality and the requirement $\gamma-1\geq \beta$ required for a Buser-type inequality, we conclude that there does not exist a $\beta$ such that $(1,\beta,2)$ satisfies both Cheeger-\ and Buser-type inequalities.

Thus, we note that if we want a Cheeger and Buser type inequality
to hold even for $1$-dimensional $1$-Lipschitz functions, we
at least require

\[ \frac{\alpha + \gamma}{2} \leq \beta \]
and
\[ \gamma - 1 \geq \beta \]

These two inequalities have solutions if and only if $\gamma - 2 \geq \alpha$, in
which case $\beta$ can occupy any number between $\frac{\alpha +
\gamma}{2}$ and $\gamma -1$.

Note that in the special case of $(1, 1, 1)$, we have no hope of 
proving any inequality of the form $\lambda_2 < O(\Phi^p)$.
Similar estimates can be derived for the \textit{smooth} weight
function $\rho(x)=\sqrt{x^2+\epsilon^2}$. The calculations are more involved, however. One can similarly show that
\begin{align*}
\Phi \asymp \epsilon,\qquad
\lambda_2\gtrsim\frac{1}{\ln(1/\epsilon)}
\end{align*}
Thus combining these bounds, we deduce
\begin{align*}
  \frac{\lambda_2}{\Phi} > \Omega(\frac{1}{\ln(1/\epsilon))}
\end{align*}
which diverges to infinity as $\epsilon\to 0$. 
As before we conclude that, for the choice
$(\alpha,\beta,\gamma)=(1,1,1)$, there is no hope of proving an
inequality of the form $\lambda_2\leq  O(\Phi^{p})$ for any $p>0$ and any class of $\rho$ containing the smooth Lipschitz weights $\sqrt{x^2+\epsilon^2}$.

\section{Problems with Existing Spectral Cut
Methods}\label{sec:counterexample}

In this section, we introduce a simple Lipschitz distribution
where the $(\alpha=1, \gamma=2)$-spectral sweep cut fails to find a
$(1, \beta)$ sparse cut for any $0 \leq  \beta < 10$. Meanwhile, the
$(1,3)$-spectral sweep cut finds a desirable cut with good
$(1,2)$-sparsity. We note that $(\alpha=1, \beta > 10)$-sparse cuts are
likely to find cuts where one side has extremely small probability mass,
making it undesirable for machine learning. 

We note that this
section combined with Theorem~\ref{thm:sweep-cut} and
Lemma~\ref{lem:converse} shows that no Cheeger and Buser
inequality can hold when $\alpha = 1$ and $\gamma = 2$ for any
$\beta$: this section combined with Theorem~\ref{thm:sweep-cut} will
show that the Cheeger-Buser inequalities can only hold for $\beta > 10$,
while Lemma~\ref{lem:converse} shows that they can only hold for
$\beta \leq 1$. Therefore, the Cheeger-Buser inequalities cannot hold for
any $\beta$, for $\alpha = 1$ and $\gamma = 2$.


\begin{theorem}\label{thm:counterexample} \textbf{$(\alpha=1,\gamma=2)$-Spectral Sweep Cut Counterexample:}

  For a $1$-Lipschitz positive valued function $\rho$, let $\Phi$ be the
  sparsity of the $(1,3)$-spectral sweep cut, and let $\Phi_{OPT}$ be
  the cut of optimal $(1, \beta)$ sparsity for any $\beta < 10$. There
  there exists a $1$-Lipschitz density function $\rho$ such that:

  \[\Phi > C \max(\Phi_{OPT}, \sqrt{\Phi_{OPT}})\]
  for any constant $C$.
\end{theorem}

\subsection{Our density function}
We first construct our $1$-Lipschitz Density function for which a
$(1,2)$ spectral cut has poor $(1,\beta)$ sparsity. Our density function
has parameters $X, Y, \epsilon, n$ which we will set later.

\begin{definition} Let $\rho: [-X, X] \times [-Y, Y] \rightarrow \mathbb{R}$ be a density function such that:

  \[ \rho(x, y) = \min(\epsilon + x, 1/n) \] 
\end{definition}
To turn this into a $1$-Lipschitz probability density function, we simply
extend it to a function $\rho':\mathbb{R}^2 \rightarrow \mathbb{R}$
where $\rho'$ agrees with $\rho$ on $[-X, X] \times [-Y, Y]$, and the
function goes $1$-Lipschitzly to $0$ outside this range.

\begin{figure}
\centering
\includegraphics[width=4.5in]{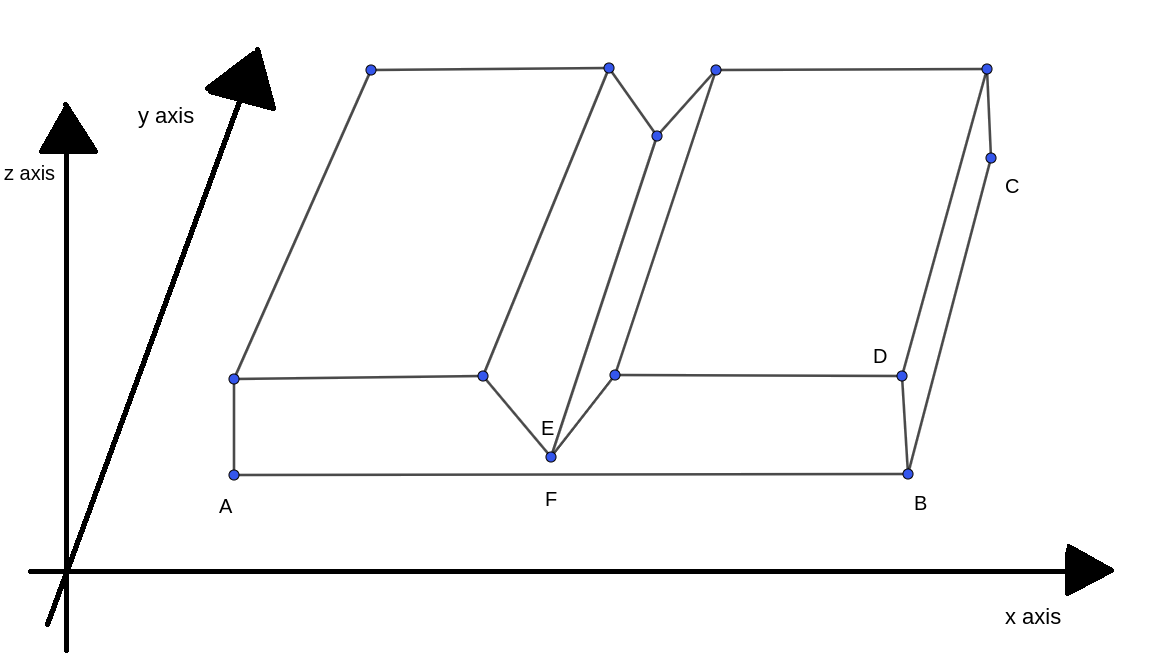}
\caption{
  The probability density function where $\rho(x, y) =
    \min(\epsilon + x, \frac{1}{n})$ for arbitrary $X, Y, n$.
      Here, $\rho(x,y)$ is plotted in the z axis, and $E$ is at point
      $(0, -Y, \epsilon)$.
 }
\label{fig:counterexample}
\end{figure}

We will set $X = \sqrt{n}/10, Y = 10\sqrt{n}$, and $n$ large, to obtain
a density function where the $(1, 2)$ spectral cut has arbitrarily bad
$(1,\beta)$ sparsity for all $\beta < 10$. 

\subsection{Proof Overview}

First, we prove theorems about the zero-set of this density's
$(\alpha=1,\gamma=2)$ eigenfunction. In particular, the zero-set of this
eigenfunction must cut from the line $x=-X$ to $x=X$. It cannot cut from
the line $y = -Y$ to the line $y = Y$.

We prove that any level-set of the eigenfunction can't cut from $y
= -Y$ to $y = Y$. We then show that any cut that doesn't cut from $y =
-Y$ to $y = Y$ has bad $(1,\beta)$ sparsity for $\beta < 10$. This
completes our proof. Moreover, any cut that doesn't cut from $y=-Y$ to
$y=Y$ is intuitively a poor cut of our density function, according to
standard machine learning intuition.

We note the natural cut of this distribution is the straight line cut $x =
0$, which the $(1,3)$-spectral sweep cut will find (this is an artifact
of our proof, though we do not explicitly prove it here).

First, we prove a few lemmas on the zero-set of the $(1,2)$
eigenfunction.
\subsection{The Zero-set of a principal $(1,2)$ eigenfunction is the
line $y = 0$} 

\begin{theorem} \label{thm:zero-set} The Zero-set of the eigenfunction for our given density
  function, is the line $y = 0$.
\end{theorem}
\begin{lemma} Let $f$ be any eigenfunction of our given density function,
  for which $f(x,y) \not= f(x, y')$ for some $x, y \not= y'$. Then
  
  \[ \int_0^Y f(x,y) dy  = 0. \]
\end{lemma}

\begin{lemma}\label{lem:symmetry} There exists a principal eigenfunction $f_2$ of our given
  density function, for which 

  \[f_2(x,y) = f_2 (-x, y) = -f_2(x,-y)\]

\end{lemma}

\begin{proof} This follows from a (non-trivial) symmetrization argument put forward in
  the graph case in Guattery and Miller~\cite{GuMi95}.
\end{proof}

\begin{lemma}\label{lem:nodal} (Nodal domains for Densities) Every principal eigenfunction $f_2$ of our given density
  function satisfies: the closure of the set $\{S = (x,y) | f_2(x,y)
  > 0\}$ is connected.
\end{lemma}
\begin{proof} This follows analogously to the proof of Fiedler's nodal
  domains for eigenfunctions of a graph~\cite{Fiedler73}.
\end{proof}

  \begin{lemma}  \label{lem:pos-neg} Let $f$ be a $(\alpha, \beta)$ eigenfunction of any density function
  supported on a compact set $S \subset \mathbb{R}^n$ for some $n$.
  For every point in the zero-set, if any open set containing that point
  contains a positive element, it must also contain a negative element.
\end{lemma}
\begin{proof}
  This follows directly from the definition of eigenfunction.
\end{proof}


  \begin{proof} (of Theorem~\ref{thm:zero-set}):
    First, we note that there is a principal eigenfunction whose zero
    set contains $y = 0$, by Lemma~\ref{lem:symmetry}. We claim there is a principal
    eigenfunction for which this is the entire zero-set. This follows
    from Lemma~\ref{lem:nodal} and Lemma~\ref{lem:pos-neg}.
  \end{proof}
  \tim{PICTURE}


\subsection{Any spectral sweep cut has high $(1,\beta)$ sparsity}
In this section, we prove that the spectral sweep cut must have high $(1,\beta)$-sparsity
for $0 < \beta < 10$, and for $\beta>10$ the spectral sweep-cut either
has high $(1,\beta)$ sparsity or else divides the probability density
into two pieces, one of which has less than $\leq 1/n$ fraction
of the probability mass. 

\begin{lemma}~\label{lem:not-vertical} Any spectral sweep cut (of the principal $(1,2)$
  eigenfunction whose eigenvector's zero-set is the line $y = 0$) can't cut through $y=Y$ and $y =
  -Y$.
\end{lemma}
\begin{proof} 
  This is clear.
\end{proof}
\begin{lemma}\label{lem:notsparse} Any cut that doesn't cut through both $y=Y$ and $y=-Y$ has
  poor $(1,\beta)$ sparsity for any $0 < \beta < 10$. For $\beta > 10$,
  a cut of good $(1,\beta)$ sparsity must have its smaller side contain $o_n(1)$ fraction of the
  mass. 
  
  To be precise, if $\Phi_{\beta}$ is the optimal $(1,\beta)$
  sparsity of the cut, and $\Phi$ is the $(1,\beta)$ sparsity induced by
  a cut that doesn't cut both $y=Y$ and $y=-Y$, then there is no
  constant $C$ independent of $n$ for which

  $\Phi^2 <  C \Phi_\beta $.
\end{lemma}
We note that Theorem~\ref{thm:counterexample} follows from
Lemma~\ref{lem:not-vertical} and~\ref{lem:notsparse}. Thus, it remains
to show Lemma~\ref{lem:notsparse}.

\begin{proof} (of Lemma~\ref{lem:notsparse}). We split this into two
  cases. Consider the side of the level set cut with smaller probability
  mass. The first case is when this side has at least half its
  probability mass outside the region $|x| < 1/n-\epsilon$. The second case is when
  the side has less than half its mass in this region.

  In the first case, we note that we can lower bound the cut by its
  projection onto the $x$ axis. A quick calculation shows that when $X =
  \frac{1}{10\sqrt{n}}$ and $Y = \frac{10}{\sqrt{n}}$, the $(1,
  \beta)$-sparsity of this cut is within a factor of $2$ of the $(1,
  \beta)$-sparsity of the cut $y = 0$ through the uniform distribution of height
  $\frac{1}{n}$ supported on $[-X, X] \times [-Y, Y]$. This $(1, \beta)$
  sparsity is 
  \[ A:= O(\frac{X}{n^\beta}) = O(\frac{\sqrt{n}}{n^\beta})
  \].

  When $\epsilon$ is chosen to be $\frac{1}{n^2\sqrt{n}}$, then the
  $(1,\beta)$ sparsity of the optimal cut is the cut $x=0$, which has
  $(1,\beta)$ sparsity of:
  \[B:= O(\frac{Y}{(n\sqrt{n})^\beta}) =
  O(\frac{\sqrt{n}}{(n^2\sqrt{n})^\beta}.\]

  We note that this choice of $\epsilon$ is the minimum such choice such
  that the principal eigenvector is not constant on the $Y$ axis.

  Now we note that $A^2/B$ goes to infinity as $n$ gets large, if and
  only if

  \[ n^{2\beta} \sqrt{n}^\beta n / n^{2\beta} \sqrt{n} \] goes to infinity,

  or 
  \[ \sqrt{n} (\sqrt{n}^\beta)\] goes to infinity. This is true for any
  $\beta > 0$. This proves Theorem~\ref{thm:counterexample} in case $1$,
  where at least half of the probability mass is outside the region $|x|
  < 1/n$.

  In case $2$, we consider the  case when the smaller side of the cut
  has more than half its probability mass inside the region $|x| < 1/n
  -\epsilon$, which we note is a very small portion of the probability
  mass of the overall probability density. In
  this case, it turns out that we need $\beta < 10$ to give isoperimetry
  guarantees, since for any $\beta > 10$, it turns out that even cuts
  containing small probability mass are considered to have good
  $(1,\beta)$ sparsity, since for large $\beta$, $(1,\beta)$ sparse cuts
  tremendously favor small cuts, even if the smaller side has negligible
  probability mass.

  Since at least half the mass is inside the
  region $|x| < 1/n$, we can assume without loss of generality that the
  entire probability mass of the smaller side of the cut is inside this region, by simply projecting
  the cut onto this region (reducing its $\beta$-perimeter while
  decreasing probability mass by at most a factor of $2$). We can again
  use a symmetry argument analogous to~\ref{lem:symmetry} to show that
  any level set of this principal eigenfunction is symmetric about the
  $x$ axis (we note Lemma~\ref{lem:notsparse} is slightly stronger than
  this as it does not assume symmetry, but for our purposes we can
  strictly deal with symmetric cuts, and the non-symmetric case follows
  through a similar argument). 

  Now given the cut is symmetric about the $x$ axis, if the cut cuts
  through $(x',y')$, then it also cuts through $(-x',y')$, and we can lower
  bound the probability mass contained by the cut $y = y'$ with $x'
  \cdot \rho(x',y')$. A simple calculation using this estimate finishes
  the proof for us.
  \tim{Add some pictures!}

\end{proof}

\section{Conclusion and Future Directions}\label{sec:conclusion}
We define a new notion of spectral sweep cuts, eigenvalues,
Rayleigh quotients, and sparsity for probability densities. We present the first known Cheeger and Buser inequality on
probability density functions, and use this to show an $(1, 3)$ spectral
sweep cut on a $L$-Lipschitz probability density function has
provably low $(1,2)$-sparsity. This work is the first spectral sweep
cut algorithm on Lipschitz densities with any guarantees on the cut quality.

Further, we show that existing spectral sweep cut methods (such
as those implicit in spectral clustering) compute
$(1, 1)$ or $(1,2)$ spectral sweep cuts, neither of which
has any sparsity guarantees. We prove that $(1,2)$ spectral sweep
cuts, which are implicitly used in traditional spectral
clustering, can lead to undesirable partitions of simple
$1$-Lipschitz probability densities. We also show that Cheeger
and Buser's inequality for probability density relies on a
careful setting of three parameters: $\alpha, \beta$, and $\gamma$.

For future directions, we conjecture that $\beta =
\alpha+1$ and $\gamma = \alpha+2$ is the only settings of
$(\alpha, \beta, \gamma)$ in which both Cheeger and Buser
inequalities are provable. This would be a stronger theorem than
we currently have for Lemma~\ref{lem:converse}, which shows
that $\gamma -2 \leq \alpha$ and $\frac{\alpha + \beta}{2} \leq
\gamma$ is required.  

In the Buser inequality, we would like to iron out the exact dimensional
dependence on the dimension,
$d$ (Theorem~\ref{thm:buser_n}). The
authors believe that this dependence can be
reduced to $\sqrt{d}$. It is an open question whether
\textit{any}
dimension dependence is required.  In particular, the latest
version of
Buser's inequality for manifolds has no dimension
dependence~\cite{ledoux2004spectral}, and this
can be proven through the heat kernels, the Bochner formula, and the Li-Yau
inequality in manifold theory. It is an open question how to
generalize their techniques into the distribution setting, as the
Bochner formula does not easily generalize to distributions.

We also conjecture that our algorithm,
\textsc{1,3-SpectralClustering}, converges to the
$(1,3)$-spectral sweep cut of $\rho$ as the number of samples
drawn from $\rho$ grows large. This would be analagous to the results
of Slepcev and Trillos on standard spectral clustering~\cite{TrillosVariational15}.

Another open question is whether multi-way Cheeger and Buser inequalities
can be proven on distributions, mirroring the work on
graphs~\cite{Louis12, kw16, LeeMultiway14, Lee2014}. This would
allow our clustering algorithms to generalize into $k$-way
clusterings.

Finally, we would like to know whether Buser and Cheeger inequalities may exist for $L$-Lipschitz
probability densities supported on manifolds with bounded
curvature. If true, this would fully generalize the work
of Cheeger and Buser on manifolds, which may lead to deeper
insight into manifold theory. Moreover, it could have
foundational impact:
a fundamental assumption underlying modern machine learning is
that most data comes from probability density supported on a manifold, and a Cheeger and Buser inequality in this
setting would give provable sparsity guarantees about
spectral sweep cuts in this setting.

\bibliographystyle{alpha}
\bibliography{123}

\begin{thebibliography}{MWW18}

\bibitem[AM84]{AlonM84}
N.~Alon and V.~D. Milman.
\newblock Eigenvalues, expanders and superconcentrators.
\newblock In {\em Proceedings of the 25th Annual Symposium on Foundations of
  Computer Science, 1984}, FOCS '84, pages 320--322, Washington, DC, USA, 1984.
  IEEE Computer Society.

\bibitem[AP09]{Andersen09}
Reid Andersen and Yuval Peres.
\newblock Finding sparse cuts locally using evolving sets.
\newblock In {\em Proceedings of the forty-first annual ACM symposium on Theory
  of computing}, STOC '09, pages 235--244. ACM, 2009.

\bibitem[ARV04]{arv04}
Sanjeev Arora, Satish Rao, and Umesh~V. Vazirani.
\newblock Expander flows, geometric embeddings and graph partitioning.
\newblock In L{\'{a}}szl{\'{o}} Babai, editor, {\em Proceedings of the 36th
  Annual {ACM} Symposium on Theory of Computing (STOC), Chicago, IL, USA, June
  13-16, 2004}, pages 222--231. {ACM}, 2004.

\bibitem[BC17]{bc17focs}
Tugkan Batu and Cl{\'{e}}ment~L. Canonne.
\newblock Generalized uniformity testing.
\newblock In Chris Umans, editor, {\em 58th {IEEE} Annual Symposium on
  Foundations of Computer Science, {FOCS} 2017, Berkeley, CA, USA, October
  15-17, 2017}, pages 880--889. {IEEE} Computer Society, 2017.

\bibitem[Bis06]{bishopBook}
Christopher~M. Bishop.
\newblock {\em Pattern Recognition and Machine Learning}.
\newblock Springer, 2006.

\bibitem[BN04]{belkin2004semisup}
Mikhail Belkin and Partha Niyogi.
\newblock Semi-supervised learning on riemannian manifolds.
\newblock {\em Mach. Learn.}, 56(1-3):209--239, June 2004.

\bibitem[Bus82]{Buser82}
Peter Buser.
\newblock A note on the isoperimetric constant.
\newblock {\em Annales scientifiques de l'\'Ecole Normale Sup\'erieure}, Ser.
  4, 15(2):213--230, 1982.

\bibitem[CGR05]{chawla05sparse}
Shuchi Chawla, Anupam Gupta, and Harald R{\"{a}}cke.
\newblock Embeddings of negative-type metrics and an improved approximation to
  generalized sparsest cut.
\newblock In {\em Proceedings of the Sixteenth Annual {ACM-SIAM} Symposium on
  Discrete Algorithms, {SODA} 2005, Vancouver, British Columbia, Canada,
  January 23-25, 2005}, pages 102--111. {SIAM}, 2005.

\bibitem[Che70]{Cheeger70}
Jeff Cheeger.
\newblock A lower bound for the smallest eigenvalue of the laplacian.
\newblock In {\em Problems in Analysis: A Symposium in Honor of Salomon
  Bochner}, pages 195--199. Princeton Univ. Press, Princeton, N. J., 1970.

\bibitem[Chu97]{ChungBook97}
F.R.K. Chung.
\newblock {\em Spectral Graph Theory}, volume~92 of {\em Regional Conference
  Series in Mathematics}.
\newblock American Mathematical Society, 1997.

\bibitem[DFK91]{Dyer91}
Martin Dyer, Alan Frieze, and Ravi Kannan.
\newblock A random polynomial-time algorithm for approximating the volume of
  convex bodies.
\newblock {\em J. ACM}, 38(1):1--17, January 1991.

\bibitem[DKN10]{dkn09}
Ilias Diakonikolas, Daniel~M. Kane, and Jelani Nelson.
\newblock Bounded independence fools degree-2 threshold functions.
\newblock In {\em 51th Annual {IEEE} Symposium on Foundations of Computer
  Science, {FOCS} 2010, October 23-26, 2010, Las Vegas, Nevada, {USA}}, pages
  11--20, 2010.

\bibitem[EKSX96]{e96}
Martin Ester, Hans-Peter Kriegel, Jorg Sander, and Xiaowei Xu.
\newblock A density-based algorithm for discovering clusters a density-based
  algorithm for discovering clusters in large spatial databases with noise.
\newblock In {\em Proceedings of the Second International Conference on
  Knowledge Discovery and Data Mining}, KDD'96, pages 226--231. AAAI Press,
  1996.

\bibitem[Fie73]{Fiedler73}
Miroslav Fiedler.
\newblock Algebraic connectivity of graphs.
\newblock {\em Czechoslovak Math. J.}, 23(98):298--305, 1973.

\bibitem[Fri44]{f44}
Kurt~Otto Friedrichs.
\newblock The identity of weak and strong extensions of differential operators.
\newblock {\em Transactions of the American Mathematical Society}, 55:132--151,
  1944.

\bibitem[GLR18]{Ge2018}
Rong Ge, Holden Lee, and Andrej Risteski.
\newblock Beyond log-concavity: Provable guarantees for sampling multi-modal
  distributions using simulated tempering langevin monte carlo.
\newblock In {\em Proceedings of the 32Nd International Conference on Neural
  Information Processing Systems}, NIPS'18, pages 7858--7867, USA, 2018. Curran
  Associates Inc.

\bibitem[GM83]{Gromov83}
Mikhail Gromov and Vitali~D Milman.
\newblock A topological application of the isoperimetric inequality.
\newblock {\em American Journal of Mathematics}, 105(4):843--854, 1983.

\bibitem[GM95]{GuMi95}
Stephen Guattery and Gary~L. Miller.
\newblock On the performance of the spectral graph partitioning methods.
\newblock In {\em SODA'95}, pages 233--242. ACM-SIAM, 1995.

\bibitem[GS06]{grady2006isoperimetric}
Leo Grady and Eric~L Schwartz.
\newblock Isoperimetric graph partitioning for image segmentation.
\newblock {\em IEEE transactions on pattern analysis and machine intelligence},
  28(3):469--475, 2006.

\bibitem[GTS15]{TrillosRate15}
Nicol{\'a}s Garc{\'i}a~Trillos and Dejan Slep{\v{c}}ev.
\newblock On the rate of convergence of empirical measures in
  $\infty$-transportation distance.
\newblock {\em Canadian Journal of Mathematics}, 67(6):1358–1383, 2015.

\bibitem[GTS16]{TrillosContin16}
Nicol{\'a}s Garc{\'i}a~Trillos and Dejan Slep{\v{c}}ev.
\newblock Continuum limit of total variation on point clouds.
\newblock {\em Archive for Rational Mechanics and Analysis}, 220(1):193--241,
  04 2016.

\bibitem[HK00]{hk00}
Piotr Haj{\l}asz and Pekka Koskela.
\newblock {\em Sobolev met poincar{\'e}}, volume 688.
\newblock American Mathematical Soc., 2000.

\bibitem[KK10]{kk10clustering}
A.~{Kumar} and R.~{Kannan}.
\newblock Clustering with spectral norm and the k-means algorithm.
\newblock In {\em 2010 IEEE 51st Annual Symposium on Foundations of Computer
  Science (FOCS)}, pages 299--308, 2010.

\bibitem[KLS95]{KLS95}
Ravi Kannan, L{\'{a}}szl{\'{o}} Lov{\'{a}}sz, and Mikl{\'{o}}s Simonovits.
\newblock Isoperimetric problems for convex bodies and a localization lemama.
\newblock {\em Discrete {\&} Computational Geometry}, 13:541--559, 1995.

\bibitem[KW16]{kw16}
Robert Krauthgamer and Tal Wagner.
\newblock Cheeger-type approximation for sparsest \emph{st}-cut.
\newblock {\em {ACM} Trans. Algorithms}, 13(1):14:1--14:21, 2016.

\bibitem[Led04]{ledoux2004spectral}
Michel Ledoux.
\newblock Spectral gap, logarithmic sobolev constant, and geometric bounds.
\newblock {\em Surveys in differential geometry}, 9(1):219--240, 2004.

\bibitem[LGT14a]{LeeMultiway14}
James~R. Lee, Shayan~Oveis Gharan, and Luca Trevisan.
\newblock Multi-way spectral partitioning and higher-order cheeger
  inequalities.
\newblock {\em Journal of the ACM}, 61(6)(37), 2014.

\bibitem[LGT14b]{Lee2014}
James~R Lee, Shayan~Oveis Gharan, and Luca Trevisan.
\newblock Multiway spectral partitioning and higher-order cheeger inequalities.
\newblock {\em Journal of the ACM (JACM)}, 61(6):37, 2014.

\bibitem[LR88]{leightonrao88}
Frank~Thomson Leighton and Satish Rao.
\newblock An approximate max-flow min-cut theorem for uniform multicommodity
  flow problems with applications to approximation algorithms.
\newblock In {\em 29th Annual Symposium on Foundations of Computer Science,
  White Plains, New York, USA, 24-26 October 1988}, pages 422--431. {IEEE}
  Computer Society, 1988.

\bibitem[LRTV12]{Louis12}
Anand Louis, Prasad Raghavendra, Prasad Tetali, and Santosh Vempala.
\newblock Many sparse cuts via higher eigenvalues.
\newblock In {\em Proceedings of the forty-fourth annual ACM symposium on
  Theory of computing}, FOCS 2012, pages 1131--1140. ACM, 2012.

\bibitem[LS90]{Lovasz90}
L.~Lovasz and M.~Simonovits.
\newblock The mixing rate of markov chains, an isoperimetric inequality, and
  computing the volume.
\newblock In {\em Proceedings of the 31st Annual Symposium on Foundations of
  Computer Science (FOCS)}, FOCS '90, pages 346--354 vol. 1, Washington, DC,
  USA, 1990. IEEE Computer Society.

\bibitem[LV18a]{Lee18survey}
Yin~Tat Lee and Santosh~S. Vempala.
\newblock The {K}annan-{L}ovász-{S}imonovits conjecture, 2018.

\bibitem[LV18b]{Lee18}
Yin~Tat Lee and Santosh~S. Vempala.
\newblock Stochastic localization + stieltjes barrier = tight bound for
  log-sobolev.
\newblock In {\em Proceedings of the 50th Annual ACM SIGACT Symposium on Theory
  of Computing}, STOC 2018, pages 1122--1129, New York, NY, USA, 2018. ACM.

\bibitem[LW01]{lw01}
Chun Liu and Noel~J Walkingron.
\newblock An eulerian description of fluids containing visco-elastic particles.
\newblock {\em Archive for rational mechanics and analysis}, 159(3):229--252,
  2001.

\bibitem[Mad10]{madry10fast}
Aleksander Madry.
\newblock Fast approximation algorithms for cut-based problems in undirected
  graphs.
\newblock In {\em 2010 IEEE 51st Annual Symposium on Foundations of Computer
  Science}, pages 245--254. IEEE, 2010.

\bibitem[Mil09]{Milman09}
Emanuel Milman.
\newblock On the role of convexity in isoperimetry, spectral gap and
  concentration.
\newblock {\em Inventiones mathematicae}, 177(1):1--43, 2009.

\bibitem[Mon03]{m03}
Peter Monk.
\newblock {\em Finite Element Methods for Maxwell's Equations (Numerical
  Analysis and Scientific Computation Series)}.
\newblock Oxford Press, 01 2003.

\bibitem[Muc72]{muckenhoupt1972hardy}
Benjamin Muckenhoupt.
\newblock Hardy's inequality with weights.
\newblock {\em Studia Mathematica}, 44(1):31--38, 1972.

\bibitem[MV10]{mv10}
Ankur Moitra and Gregory Valiant.
\newblock Settling the polynomial learnability of mixtures of gaussians.
\newblock In {\em 51th Annual {IEEE} Symposium on Foundations of Computer
  Science, {FOCS} 2010, October 23-26, 2010, Las Vegas, Nevada, {USA}}, pages
  93--102, 2010.

\bibitem[MWW18]{MillerHardy18}
Gary~L. Miller, Noel~J. Walkington, and Alex~L. Wang.
\newblock Hardy-muckenhoupt bounds for laplacian eigenvalues.
\newblock {\em CoRR}, abs/1812.02841, 2018.

\bibitem[NJW01]{NgSpectral01}
Andrew~Y. Ng, Michael~I. Jordan, and Yair Weiss.
\newblock On spectral clustering: Analysis and an algorithm.
\newblock In {\em Proceedings of the 14th International Conference on Neural
  Information Processing Systems: Natural and Synthetic}, NIPS'01, pages
  849--856, Cambridge, MA, USA, 2001. MIT Press.

\bibitem[OSVV08]{Orecchia08}
Lorenzo Orecchia, Leonard~J Schulman, Umesh~V Vazirani, and Nisheeth~K Vishnoi.
\newblock On partitioning graphs via single commodity flows.
\newblock In {\em Proceedings of the fortieth annual ACM symposium on Theory of
  computing}, STOC '08, pages 461--470. ACM, 2008.

\bibitem[OV11]{Orecchia2011}
Lorenzo Orecchia and Nisheeth~K. Vishnoi.
\newblock Towards an sdp-based approach to spectral methods: A
  nearly-linear-time algorithm for graph partitioning and decomposition.
\newblock In {\em Proceedings of the Twenty-second Annual ACM-SIAM Symposium on
  Discrete Algorithms}, SODA '11, pages 532--545, Philadelphia, PA, USA, 2011.
  Society for Industrial and Applied Mathematics.

\bibitem[RBV10]{rosasco2010learning}
Lorenzo Rosasco, Mikhail Belkin, and Ernesto~De Vito.
\newblock On learning with integral operators.
\newblock {\em Journal of Machine Learning Research}, 11(Feb):905--934, 2010.

\bibitem[SBC14]{SuNesterov14}
Weijie Su, Stephen Boyd, and Emmanuel~J. Cand\`{e}s.
\newblock A differential equation for modeling nesterov’s accelerated
  gradient method: Theory and insights.
\newblock In {\em Proceedings of the 27th International Conference on Neural
  Information Processing Systems - Volume 2}, NIPS’14, page 2510–2518,
  Cambridge, MA, USA, 2014. MIT Press.

\bibitem[SM97]{ShiMalik97}
J.B. Shi and Jitendra Malik.
\newblock Normalized cuts and image segmentation.
\newblock {\em IEEE Trans. on Pattern Anal. and Mach. Intell.}, 22:888--905, 01
  1997.

\bibitem[Sob38]{s38}
S~Soboleff.
\newblock Sur un th{\'e}or{\`e}me d'analyse fonctionnelle.
\newblock {\em Matematicheskii Sbornik}, 46(3):471--497, 1938.

\bibitem[SS09]{ss09}
Elias~M Stein and Rami Shakarchi.
\newblock {\em Real analysis: measure theory, integration, and Hilbert spaces}.
\newblock Princeton University Press, 2009.

\bibitem[ST04]{SpielmanTeng2004}
Daniel~A. Spielman and Shang-Hua Teng.
\newblock Nearly-linear time algorithms for graph partitioning, graph
  sparsification, and solving linear systems.
\newblock In {\em Proceedings of the Thirty-sixth Annual ACM Symposium on
  Theory of Computing}, STOC '04, pages 81--90, New York, NY, USA, 2004. ACM.

\bibitem[ST07]{SPIELMAN2007284}
Daniel~A. Spielman and Shang-Hua Teng.
\newblock Spectral partitioning works: Planar graphs and finite element meshes.
\newblock {\em Linear Algebra and its Applications}, 421(2):284 -- 305, 2007.
\newblock Special Issue in honor of Miroslav Fiedler.

\bibitem[SW18]{sw18}
Christian Sohler and David~P. Woodruff.
\newblock Strong coresets for k-median and subspace approximation: Goodbye
  dimension.
\newblock In {\em 59th {IEEE} Annual Symposium on Foundations of Computer
  Science, (FOCS) 2018, Paris, France, October 7-9, 2018}, pages 802--813,
  2018.

\bibitem[SW19]{sw19expander}
Thatchaphol Saranurak and Di~Wang.
\newblock Expander decomposition and pruning: Faster, stronger, and simpler.
\newblock In Timothy~M. Chan, editor, {\em Proceedings of the Thirtieth Annual
  {ACM-SIAM} Symposium on Discrete Algorithms, {SODA} 2019, San Diego,
  California, USA, January 6-9, 2019}, pages 2616--2635. {SIAM}, 2019.

\bibitem[TS15]{TrillosVariational15}
Nicolás~García Trillos and Dejan Slepčev.
\newblock A variational approach to the consistency of spectral clustering,
  2015.

\bibitem[VL07]{von2007tutorial}
Ulrike Von~Luxburg.
\newblock A tutorial on spectral clustering.
\newblock {\em Statistics and computing}, 17(4):395--416, 2007.

\bibitem[VLBB08]{von2008consistency}
Ulrike Von~Luxburg, Mikhail Belkin, and Olivier Bousquet.
\newblock Consistency of spectral clustering.
\newblock {\em The Annals of Statistics}, pages 555--586, 2008.

\bibitem[VSCC08]{vsc08}
Nicholas~T Varopoulos, Laurent Saloff-Coste, and Thierry Coulhon.
\newblock {\em Analysis and geometry on groups}, volume 100.
\newblock Cambridge university press, 2008.

\bibitem[Whi17]{w17}
Jonathan Whiteley.
\newblock Linear elliptic partial differential equations.
\newblock In {\em Finite Element Methods}, pages 119--141. Springer, 2017.

\bibitem[WN17]{wulff17expander}
Christian Wulff-Nilsen.
\newblock Fully-dynamic minimum spanning forest with improved worst-case update
  time.
\newblock In {\em Proceedings of the 49th Annual ACM SIGACT Symposium on Theory
  of Computing (STOC)}, STOC 2017, page 1130–1143, New York, NY, USA, 2017.
  Association for Computing Machinery.

\bibitem[WWJ16]{WibisonoGradient16}
Andre Wibisono, Ashia~C. Wilson, and Michael~I. Jordan.
\newblock A variational perspective on accelerated methods in optimization.
\newblock {\em Proceedings of the National Academy of Sciences of the United
  States of America}, 113 47:E7351--E7358, 2016.

\end{thebibliography}
\clearpage
\begin{appendices}
  \section{Cheeger and Buser for Density Functions does not easily follow
  from Graph or Manifold Cheeger and Buser}\label{app:notgraph}

\subsection{Comments on Graph Cheeger-Buser}
The most natural method of proving distributional Cheeger-Buser
inequality using the graph Cheeger-Buser inequality is to
generate a vertex and edge weighted graph approximating the distribution, and write down
graph Cheeger-Buser. Then, one would generate a sequence of graphs with
an increasing number of vertices. Ideally, the graph Cheeger-Buser inequality
on these graphs would converge to a
Cheeger-Buser inequality on the underlying distribution. This
discretization approach follows a standard paradigm of
approximating distributions with graphs, present in
numerical methods, finite element methods,
         and machine learning
         ~\cite{TrillosRate15,TrillosVariational15,SPIELMAN2007284}.

Such an approach cannot work (no matter how the
    eigenvalues and isoperimetric cuts are defined for
    distributions). The easiest way to see this is to attempt to execute
    this strategy for a simple uniform distribution in $1$ dimension, on
    the interval $[0,1]$. One would naively approximate this
    distribution
    with a line graph with $n$ vertices, with edge weights $w_n$ and vertex
    weights $m_n$. Then one would take $n$ to go to infinity.

    If one writes down the Cheeger and Buser inequalities for graphs in
    this example, we get:
    
    \[\frac{w_n}{m_n n^2} \leq \Phi_{OPT} \leq \frac{w_n}{m_n n} \]

    No matter what $m_n$ and $w_n$ are, the ratio between the upper and
    lower bound is $n$, which diverges. Thus, either the Cheeger
    inequality or the Buser inequality becomes meaningless: either the
    lower bound goes to $0$ or the upper bound goes to $\infty$, or
    both, depending on
    how $w_n$ and $m_n$ are set.

    Thus, even for the simple case of a uniform distribution on $[0,1]$
    the natural strategy for deriving probability density Cheeger/Buser
    from graph Cheeger/Buser fails.

\subsection{Comments on Manifold Cheeger-Buser}
Distributional Buser does not easily follow from an
application of
the manifold Buser inequality. We recall that manifold Buser only applies for
    manifolds with bounded Ricci curvature. The natural way to parlay manifold Buser into
distributional Buser on $\RR^d$ is to change the underlying metric tensor
on $\RR^d$ to factor in the probability density function at that
point. However, the authors are unaware of any method of doing
this for which one can recover a meaningful Cheeger and
Buser inequality. Moreover, it is unclear how to obtain any Ricci
curvature bounds when we change the metric tensor.

Most modern approaches to proving Buser's inequality for
    manifolds rely on the
Li-Yau inequality, which in turn depends on the Bochner identity
    for manifolds on bounded Ricci
    curvature~\cite{ledoux2004spectral}.
The authors are unaware of a clean Bochner-like identity for
distributions. Older techniques use Almgren's minimizing currents
and/or Epsilon nets~\cite{Buser82}. For the former, we do not know of any
analog for distributions. For the latter, the corresponding
Buser inequality has a $2^{d}$ multiplicative dependence, which
is significantly worse than our $d$ dependence.


\section{A weighted Cheeger inequality in one dimension}
\label{sec:one_dim}
\begin{theorem}
\label{thm:cheeger_1}
Let $\Omega = (a,b)$ where $-\infty<a<b<\infty$. Let $\rho:(a,b)\to\RR_{>0}$ be Lipschitz continuous. Then,
\begin{align*}
\Phi(\Omega)^{2} &\leq 4\norm{\rho^{\beta - \frac{\alpha+\gamma}{2}}}^2_\infty\lambda_2(\Omega).
\end{align*}
In particular, when $(\alpha,\beta,\gamma) = (1,2,3)$, we have
\begin{align*}
\Phi(\Omega)^2\leq 4\lambda_2(\Omega).
\end{align*}
\end{theorem}
\begin{proof}
Let $w\in W^{1,2}(\Omega) \cap C^\infty(\Omega)$ be a strictly decreasing function with $\int_\Omega\rho^\alpha w\,dx = 0$.
Let $v = w + a1$ where $a$ is chosen such that
$\abs{\set{v<0}}_\alpha = \abs{\set{v>0}}_\alpha$.
Note that
\begin{align*}
R(w) &= \frac{\int_\Omega \rho^\gamma (w')^2\,dx}{\int_\Omega \rho^\alpha w^2\,dx}\\
&\geq \frac{\int_\Omega \rho^\gamma (w')^2\,dx}{\int_\Omega \rho^\alpha w^2\,dx + a^2 \abs{\Omega}_\alpha}\\
&= R(v).
\end{align*}
Let $\hat x\in(a,b)$ be the unique value such that $v(\hat x) = 0$. Without loss of generality, the function $u= \max(v,0)$ satisfies $R(u)\leq R(v)$ and has $u(a)= 1$.

Let $g = u^2$. Noting that $g' = 2uu'$ a.e., we can apply Cauchy-Schwarz to obtain
\begin{align*}
\int_a^{\hat x} \rho^\beta \abs{g'}\,dx
&= 2\int_a^{\hat x} \rho^\beta \abs{u}\abs{u'}\,dx\\
&\leq 2\sqrt{\int_a^{\hat x} \rho^{2\beta-\alpha} (u')^2\,dx}\sqrt{\int_a^{\hat x} \rho^\alpha u^2\,dx}\\
&\leq 2\norm{\rho^{\beta - \frac{\alpha+\gamma}{2}}}_\infty\sqrt{\int_a^{\hat x} \rho^\gamma (u')^2\,dx}\sqrt{\int_a^{\hat x} \rho^\alpha u^2\,dx}.
\end{align*}
Then, dividing by $\int_a^{\hat x} \rho^\alpha g\,dx$, we have
\begin{align*}
\frac{\int_a^{\hat x} \rho^\beta \abs{g'}\,dx}{\int_a^{\hat x} \rho^\alpha g\,dx} &\leq 2\norm{\rho^{\beta - \frac{\alpha+\gamma}{2}}}_\infty\sqrt{ R(w)}.
\end{align*}
By change of variables,
\begin{align*}
\int_a^{\hat x} \rho^\beta \abs{g'}\,dx &= \int_{0}^{1} \rho^\beta(g^{-1}(t))\,dt.
\end{align*}
Writing $g(x) = \int_0^{g(x)} 1\,dt$ and applying Tonelli's theorem, we rewrite the denominator
\begin{align*}
\int_a^{\hat x} \rho^\alpha g\,dx &= \int_{0}^{1}\abs{(a,g^{-1}(t))}_\alpha\,dt.
\end{align*}
Thus, by averaging, there exists some $t^*$ such that,
\begin{align*}
\Phi(\Omega) &\leq \frac{\rho^\beta(t^*)}{\abs{(a,t^*)}_\alpha}
\leq \frac{\int_a^{\hat x}\rho^\beta \abs{g'}\,dx}{\int_a^{\hat x}\rho^\alpha g\,dx}
\leq 2\norm{\rho^{\beta - \frac{\alpha+\gamma}{2}}}_\infty\sqrt{R(w)}.
\end{align*}g
\end{proof}

\begin{theorem}
\label{thm:buser_1}
Let $\Omega = (a,b)$ where $-\infty<a<b<\infty$. Let $\rho:(a,b)\to\RR_{>0}$ be Lipschitz continuous with Lipschitz constant $L$. Then,
\begin{align*}
\lambda_2(\Omega) &\leq 8\cdot (3/2)^{\gamma/\alpha} \norm{\rho^{\gamma -1 - \beta}}_\infty \max\left(4\norm{\rho^{\alpha+1-\beta}}_\infty \Phi^2(\Omega), \frac{\alpha}{\ln(3/2)}L\Phi(\Omega)\right).
\end{align*}
In particular, when $(\alpha,\beta,\gamma) = (1,2,3)$, we have
\begin{align*}
\lambda_2(\Omega) &\leq O\left(\max\left(\Phi^2(\Omega), L\Phi(\Omega)\right)\right).
\end{align*}
\end{theorem}
\begin{proof}
Let $\hat x\in(a,b)$. We will show that there exists a $u\in W^{1,2}(\Omega)$ with small Rayleigh quotient compared to $\Phi(\hat x)$.
Let $A = (a,\hat x)$ and $B = (\hat x, b)$. Without loss of generality $\abs{A}_\alpha \leq \abs{B}_\alpha$ and hence $\Phi(\hat x) = \frac{\rho^\beta (\hat x)}{\abs{A}_\alpha}$. For notational convenience, we will write $\Phi = \Phi(\hat x)$ in this proof.

Let
\begin{align*}
u(x) = \begin{cases}
	\abs{A}_\alpha & a \leq x \leq \hat x\\
	-\abs{B}_\alpha & \hat x < x\leq b.
\end{cases}
\end{align*}
Let $\delta=\theta \rho(\hat x)$ where $\theta>0$ will be picked later. Define the continuous function
\begin{align*}
u_\delta(x) = \begin{cases}
	\abs{A}_\alpha & a \leq x \leq x_1\\
	\text{linear with slope }\frac{-\abs{\Omega}_\alpha}{\delta} & x_1 \leq x \leq x_2\\
	-\abs{B}_\alpha & x_2 \leq x \leq b
\end{cases}
\end{align*}
where $a\leq x_1<\hat x< x_2\leq b$ are picked such that $\int_a^b \rho^\alpha u_\delta\,dx = 0$. Note $x_2 - x_1\leq \delta$.

We bound the numerator in $R(u_\delta)$ using the mean value theorem.
\begin{align*}
\int_a^b \rho^\gamma (u_\delta')^2\,dx  &= \frac{\abs{\Omega}_\alpha^2}{\delta^2}\int_{x_1}^{x_2}\rho^\gamma \,dx\\
&\leq \frac{\abs{\Omega}_\alpha^2}{\delta}\rho^\gamma(\tilde x)\hspace{2em}\text{for some }\tilde x \in[x_1,x_2]\\
&\leq \abs{\Omega}_\alpha^2\rho^{\gamma-1}(\hat x)(1+L\theta)^\gamma/\theta
\end{align*}
In the third line we used the Lipschitz estimate $\rho(\tilde x) \leq \rho(\hat x)(1+L\theta)$.
We lower bound the denominator in $R(u_\delta)$ using the mean value theorem and the same Lipschitz estimate. We will also recall that $\Phi = \rho^\beta(\hat x)/\abs{A}_\alpha$.
\begin{align*}
\int_a^b \rho^\alpha u_\delta^2\,dx &\geq \int_a^b \rho^\alpha u^2\,dx - \int_{x_1}^{x_2}  \rho^\alpha u^2\,dx\\
&\geq \abs{A}_\alpha\abs{B}_\alpha\abs{\Omega}_\alpha - \delta\rho^\alpha(\tilde x)\abs{B}_\alpha^2\hspace{2em}\text{for some }\tilde x \in[x_1,x_2]\\
&\geq \abs{A}_\alpha\abs{B}_\alpha\abs{\Omega}_\alpha - \rho^{\alpha+1}(\hat x)\abs{B}_\alpha^2(1+L\theta)^\alpha\theta\\
&\geq \abs{\Omega}_\alpha^2\left(\abs{A}_\alpha/2 - \rho^{\alpha+1}(\hat x) (1+L\theta)^\alpha \theta\right)\\
&\geq \abs{\Omega}_\alpha^2\abs{A}_\alpha\left(1/2 - \norm{\rho^{\alpha+1-\beta}}_\infty \Phi (1+L\theta)^\alpha \theta\right)
\end{align*}
The parameter $\theta$ will be chosen such that the estimate of the denominator is positive.
We combine the two bounds above.
\begin{align*}
R(u_\delta)
&\leq \frac{\abs{\Omega}_\alpha^2\rho^{\gamma-1}(\hat x)(1+L\theta)^\gamma/\theta}{\abs{\Omega}_\alpha^2\abs{A}_\alpha\left(1/2 - \norm{\rho^{\alpha+1-\beta}}_\infty \Phi (1+L\theta)^\alpha \theta\right)}\\
&= \frac{\rho^{\gamma-1-\beta}(\hat x)\Phi(1+L\theta)^\gamma/\theta}{1/2 - \norm{\rho^{\alpha+1-\beta}}_\infty \Phi (1+L\theta)^\alpha \theta}\\
&\leq \frac{\norm{\rho^{\gamma -1 - \beta}}_\infty\Phi(1+L\theta)^\gamma/\theta}{1/2 -\norm{\rho^{\alpha + 1 - \beta}}_\infty\Phi(1+L\theta)^\alpha\theta}.
\end{align*}

We make the following choice of $\theta>0$,
\begin{align*}
\theta = \min\left(\frac{1}{4\Phi\norm{\rho^{\alpha+1-\beta}}_\infty}, \frac{\ln(3/2)}{\alpha L} \right).
\end{align*}
Then, $(1+L\theta)\leq (3/2)^{1/\alpha}$ and $\Phi\theta \leq \frac{1}{4\norm{\rho^{\alpha+1-\beta}}_\infty}$. Thus,
\begin{align*}
\lambda_2 &\leq R(u_\delta)\\
&\leq 8\cdot (3/2)^{\gamma/\alpha} \norm{\rho^{\gamma -1 - \beta}}_\infty \frac{\Phi}{\theta}\\
&= 8\cdot (3/2)^{\gamma/\alpha} \norm{\rho^{\gamma -1 - \beta}}_\infty \max\left(4\norm{\rho^{\alpha+1-\beta}}_\infty \Phi^2, \frac{\alpha}{\ln(3/2)}L\Phi\right).
\end{align*}
Finally, picking $\hat x$ such that $\Phi(\hat x)\to \Phi(\Omega)$ completes the proof.
\end{proof}

\begin{remark}
Recall the example presented in Section~\ref{subsec:lipschitz_example}, i.e. $\Omega = (-1,1)$, $\rho = \abs{x}+\epsilon$. For the choice $(\alpha,\beta,\gamma) = (1,1,1)$, it was shown that $\frac{\lambda_2(\Omega)}{\Phi(\Omega)^{p}}$ diverges to infinity as $\epsilon\to 0$ for any $p>0$. This does not contradict our Theorem~\ref{thm:buser_1}, which only asserts that
\begin{align*}
\lambda_2(\Omega) &\lesssim \frac{1}{\epsilon}\max\left(\Phi^{2}(\Omega),\Phi(\Omega)\right).
\end{align*}
\end{remark}

\end{appendices}

\end{document}